\documentclass[11pt]{article}

\usepackage[T1]{fontenc}
\usepackage{lmodern}
\usepackage[margin=1in]{geometry}
\usepackage{amsmath,amssymb,amsthm,mathtools,bm}
\usepackage{booktabs}
\usepackage{array}
\usepackage{multirow}
\usepackage{graphicx}
\usepackage{wrapfig}
\usepackage{float}
\usepackage{xcolor}
\usepackage{tikz}
\usetikzlibrary{arrows.meta,calc}
\usepackage{caption}
\usepackage[numbers,compress]{natbib}
\usepackage[colorlinks=true,linkcolor=blue!55!black,citecolor=blue!55!black,urlcolor=blue!55!black]{hyperref}
\usepackage[capitalize,noabbrev]{cleveref}

\definecolor{appthmgray}{gray}{0.96}
\definecolor{appthmborder}{gray}{0.72}
\definecolor{frontierbg}{HTML}{FBFAF7}
\newsavebox{\appthmstatementbox}
\newenvironment{appthmbox}[1]{%
  \par\medskip\noindent
  \begingroup
  \setlength{\fboxsep}{6pt}%
  \setlength{\fboxrule}{0.4pt}%
  \def\appthmboxtitle{#1}%
  \begin{lrbox}{\appthmstatementbox}%
  \begin{minipage}{\dimexpr\linewidth-2\fboxsep-2\fboxrule\relax}
  \textbf{\appthmboxtitle.}\enspace\itshape
}{%
  \end{minipage}%
  \end{lrbox}%
  \fcolorbox{appthmborder}{appthmgray}{\usebox{\appthmstatementbox}}%
  \endgroup
  \par\medskip
}
\newenvironment{appthmframe}{%
  \par\medskip\noindent
  \begingroup
  \setlength{\fboxsep}{6pt}%
  \setlength{\fboxrule}{0.4pt}%
  \begin{lrbox}{\appthmstatementbox}%
  \begin{minipage}{\dimexpr\linewidth-2\fboxsep-2\fboxrule\relax}
}{%
  \end{minipage}%
  \end{lrbox}%
  \fcolorbox{appthmborder}{appthmgray}{\usebox{\appthmstatementbox}}%
  \endgroup
  \par\medskip
}

\theoremstyle{plain}
\newtheorem{theorem}{Theorem}[section]
\newtheorem{lemma}[theorem]{Lemma}
\newtheorem{proposition}[theorem]{Proposition}
\newtheorem{corollary}[theorem]{Corollary}
\theoremstyle{definition}
\newtheorem{definition}[theorem]{Definition}

\theoremstyle{remark}
\newtheorem{remark}[theorem]{Remark}

\newcommand{\calX}{\mathcal{S}}
\newcommand{\E}{\mathbb{E}}

\newcommand{\1}{\mathbf{1}}
\newcommand{\dd}{\mathrm{d}}
\newcommand{\KL}{D_{\mathrm{KL}}}
\newcommand{\pdata}{p_{\mathrm{data}}}
\newcommand{\simplex}{\Delta^{S-1}}
\newcommand{\Rhat}{\widehat{R}}
\newcommand{\Rfwd}{R}
\newcommand{\ratekl}{\mathcal{D}}
\newcommand{\master}{\mathfrak{L}}
\newcommand{\ours}{\textsc{Unifusion}}
\newcommand{\method}{\textsc{Unifusion}}
\newcommand{\proofcase}[1]{\par\smallskip\noindent\textbf{\textit{#1.}}\ }

\newcommand{\best}[1]{\textbf{#1}}

\title{\method{}: Adapting Autoregressive Language Models into\\
Discrete Diffusion under a Unified Reverse-Rate Objective}

\author{%
  \textbf{Xiaoyi Jiang}$^{1,4}$\thanks{Equal contribution.}
  \quad
  \textbf{Jingyuan Li}$^{2,3,4}$\footnotemark[1]
  \quad
  \textbf{Yixuan Jiang}$^{1,4}$
  \quad
  \textbf{Wei Liu}$^{3}$ \\[0.35ex]
  \quad
  \textbf{Yi Zhu}$^{1,2,4}$
  \quad
  \textbf{Zuoqiang Shi}$^{1,2,4}$
  \quad
  \textbf{Pipi Hu}$^{2,4}$\thanks{Corresponding author.} \\[0.5ex]
  $^{1}$Tsinghua University
  \quad
  $^{2}$Beijing Institute of Mathematical Sciences and Applications \\
  $^{3}$Wuhan University
  \quad
  $^{4}$MathonAI
}
\date{}

\begin{document}
\maketitle

\begin{abstract}
Existing methods mainly adapt pretrained autoregressive (AR) language models to masked
diffusion, whereas we directly adapt them to uniform-noise diffusion, where every token remains editable
during sampling. However, adapting AR checkpoints across corruption kernels remains challenging because
existing DLMs use different objectives and prediction parameterizations. We establish connections among
SEDD, MDLM/GIDD, M2S, and Neural CTMC by expressing their conditional losses as a single generalized
Kullback--Leibler objective over model reverse rates. We further derive conversions from clean-token
predictions to concrete-score, posterior-mean, and exit-rate/jump parameterizations, yielding a shared
\(x_0\) interface that supports switching between mask and uniform kernels. Building on these connections,
we propose \ours{}, a simple continual pre-training approach for directly adapting pretrained GPT2
checkpoints to uniform-noise diffusion. Through systematic evaluation of 124M- and 355M-parameter models,
we show that \ours{} steadily improves the trade-off between generative perplexity (GenPPL) and unigram
entropy as the sampling budget increases from 16 to 256 steps. At 256 steps, \ours{}-S and \ours{}-M
achieve GenPPL/entropy pairs of \(97.783/5.2626\) and \(71.516/5.6669\), respectively; no evaluated
model at the same scale simultaneously outperforms \ours{} on both metrics. At both scales, \ours{} also
achieves the highest WinoGrande, SIQA, and BBH accuracy among the compared diffusion models.
\end{abstract}

\section{Introduction}
\label{sec:intro}

Autoregressive (AR) language models generate one token at a time from left to right. This factorization
supports efficient training and has produced many reusable checkpoints, but the generation order is
fixed: once a token is emitted, later steps cannot revise it. Discrete diffusion instead corrupts a
complete sequence and learns to reverse that corruption~\citep{austin2021structured,campbell2022continuous}.
Recent discrete diffusion models have narrowed the quality gap to AR models
~\citep{lou2024discrete,sahoo2024simple,shi2024simplified,nie2025llada}, but training a new diffusion
model from random initialization discards the language knowledge already stored in AR checkpoints.

DiffuGPT and DiffuLLaMA reuse GPT2 and LLaMA weights, and Dream-7B starts from Qwen
~\citep{radford2019language,gong2025scaling,ye2025dream}. These conversions use absorbing mask
corruption. This choice has a direct connection to AR training: an absorbing diffusion objective
averages over token-reveal orders, and left-to-right AR decoding is one such order
~\citep{uria2014deep,hoogeboom2022autoregressive,ou2025your,shi2024simplified}. However, a mask model
can only replace positions that remain masked. In our samples, increasing the number of denoising steps
sometimes lowers generative perplexity (GenPPL) while also lowering unigram entropy and producing
repeated numbers, templates, or table fragments (\Cref{sec:exp-frontier}). Uniform corruption offers a
different target: it replaces tokens with vocabulary items rather than a special mask, so all positions
remain editable throughout sampling. The open question is how to reach this model directly from an AR
checkpoint, without first training a masked diffusion model.

A direct conversion must specify both what the network predicts and how that prediction is trained.
Existing discrete diffusion methods make different choices. SEDD predicts concrete scores; MDLM and
GIDD predict clean tokens; M2S predicts posterior means; and Neural CTMC predicts an exit rate together
with a jump distribution
~\citep{lou2024discrete,sahoo2024simple,shi2024simplified,vonrutte2025gidd,neuralctmc2026,m2s2026}.
Their losses consequently have different formulas. Comparing those formulas alone does not reveal
whether two methods learn different reverse processes or merely describe the same process with different
outputs. Moreover, concrete scores, posterior means, and jump rates all contain the forward corruption
kernel. An output built for absorbing corruption therefore cannot be reused unchanged for uniform
corruption.

We address the training question in reverse-rate space (\Cref{sec:theory}). For a fixed clean token and
noised token, a discrete diffusion model assigns a rate to every possible reverse jump. We show that,
after mapping each method's output to these rates, the conditional losses of SEDD, MDLM/GIDD, M2S, and
Neural CTMC are the same generalized KL objective. We then convert a clean-token prediction to the
other three prediction formats. This equivalence has a precise scope: converting the same clean-token
logits into the four prediction formats produces identical reverse rates and the same rate-space objective.
Native optimization can still follow different paths
because the methods weight examples differently and use different batch sizes.

We address the kernel question with the clean-token distribution \(x_\theta(x_0\mid x_t,t)\)
(\Cref{thm:conversion}). Unlike a concrete score or jump rate, this distribution always answers the
same question: which vocabulary token occupied the current position before corruption? Once it is
predicted, the known forward kernel converts it into the reverse rates for either absorbing or uniform
noise. It therefore forms a concrete interface between a pretrained AR output and multiple diffusion
kernels.

Based on this interface, \ours{} adapts GPT2 directly to uniform diffusion (\Cref{sec:method}). The AR
softmax at position \(i-1\) already predicts the clean token at position \(i\); we shift these outputs by
one position and use them to initialize the diffusion \(x_0\) predictor. During fine-tuning, we gradually
open causal attention to bidirectional attention, add time conditioning, and use a linear uniform
corruption schedule. This requires no intermediate masked checkpoint.

Our experiments answer three questions (\Cref{sec:exp}). First, rate-equivalent native objectives follow
different finite-budget trajectories from the same GPT2-small initialization. Second, direct
AR-to-uniform training attains a minimum GenPPL comparable to the two-stage masked path. Third, across
five seeds, \ours{}-S and \ours{}-M improve both GenPPL and unigram entropy as the sampling budget grows.
At 256 steps, their GenPPL/entropy pairs are \(97.783/5.2626\) and \(71.516/5.6669\); at both scales,
they also achieve the highest WinoGrande, SIQA, and BBH accuracy among the compared diffusion models.

Our contributions are:
\begin{itemize}\itemsep1pt
  \item We unify the conditional losses of SEDD, MDLM/GIDD, M2S, and Neural CTMC in reverse-rate
    space and derive conversions from clean-token predictions to all four parameterizations
    (\Cref{sec:theory}).
  \item We use clean-token \(x_0\) predictions to bridge AR initialization and absorbing/uniform
    kernels, enabling direct AR-to-uniform adaptation (\Cref{sec:method,sec:exp-shift}).
  \item We separate rate-space equivalence from finite-budget optimization and improve the
    GenPPL--entropy trade-off while retaining strong zero-shot accuracy (\Cref{sec:exp}).
\end{itemize}

\section{Preliminaries}
\label{sec:prelim}

Let \(\mathcal V=\{1,\dots,S\}\) be the vocabulary. A uniform CTMC has state space
\(\calX=\mathcal V\), whereas an absorbing CTMC uses the augmented space
\(\calX=\mathcal V\cup\{m\}\). A CTMC on \(\calX\) over \([0,T]\) is characterized by a
time-dependent rate matrix \(\Rfwd_t\), with \(\Rfwd_t(i,j)\ge0\) for \(i\neq j\) and
\(\sum_j\Rfwd_t(i,j)=0\).

\begin{definition}
\label{def:fwd-rate}
If the chain is in state \(i\) at time \(s\) and \(t=s+h\) with \(h>0\), then, as \(h\to0\),
\begin{equation}
\label{eq:infinitesimal}
  q_{t\mid s}(j\mid i)=\delta_{i,j}+\Rfwd_t(i,j)\,h+o(h),
\end{equation}
and \(\Rfwd_t\) is the \emph{forward rate matrix}.
\end{definition}

\noindent The Kolmogorov forward equation \(\dot q_t=q_t\Rfwd_t\) determines the finite-time transition
matrix \(P_{s,t}\), and below we use two common forms of the corresponding cumulative kernel.

\begin{definition}
\label{def:kernel}
The cumulative kernel \(P_{s,t}\) is the finite-time transition matrix defined by
\begin{equation}
\label{eq:kernel-ode}
  \partial_t P_{s,t}=P_{s,t}\Rfwd_t,\qquad P_{s,s}=I,\qquad
  q_{t\mid s}(j\mid i)=(P_{s,t})_{ij}.
\end{equation}
We use two closed forms. \emph{(1) Generator form.} If
\(\Rfwd_t=\beta^{\mathrm{rate}}_t\bar R\) for a fixed generator \(\bar R\), then
\[
  P_{s,t}=\exp\!\left\{\Bigl(\int_s^t\beta^{\mathrm{rate}}_u\,\dd u\Bigr)\bar R\right\},
\]
with the homogeneous case given by \(\beta^{\mathrm{rate}}_t\equiv1\), and with a target prior \(\pi\)
imposed by choosing \(\bar R\) reversible,
\(\pi(i)\bar R(i,j)=\pi(j)\bar R(j,i)\). \emph{(2) Interpolating form.} The time-\(0\)-to-\(t\) transition
is specified directly by
\begin{equation}
\label{eq:mixture}
  q_{t\mid0}(j\mid i)=\alpha_t\delta_{ij}+\beta_t\pi(j),\qquad
  P_{0,t}=\alpha_t I+\beta_t\1\pi^\top,
\end{equation}
where \(\alpha_t+\beta_t=1\), \(\alpha_t\) is differentiable and nonincreasing,
\(\alpha_0=1\), \(\alpha_T=0\), \(\alpha_t>0\) for \(t<T\), and \(p_{\mathrm{prior}}=\pi\). The classical
choices \(\pi=e_m\) on \(\mathcal V\cup\{m\}\) and \(\pi=\1/S\) on \(\mathcal V\) give masked and
uniform kernels.
\end{definition}

For \(i\neq j\), let \(j\) be the current noised token, \(i\) the reverse jump target, and
\(\Rfwd_t(i,j)\) the rate of the forward jump \(i\to j\). The data-marginal reverse rate and the
conditional reverse rate given a clean token \(x_0\) are
\[
  \Rhat_t(j,i):=\Rfwd_t(i,j)\frac{q_t(i)}{q_t(j)},\qquad
  \Rhat_t(j,i\mid x_0):=\Rfwd_t(i,j)
  \frac{q_{t|0}(i\mid x_0)}{q_{t|0}(j\mid x_0)},
\]
with \(q_t(k):=\sum_z\pdata(z)q_{t|0}(k\mid z)\). Each ratio is used when its denominator is positive;
we set the corresponding rate to \(0\) when the denominator is \(0\). The marginal model rate and the clean-token-induced
conditional rate are denoted by \(\Rhat^\theta_t(j,i)\) and \(\Rhat_t(j,i\mid x_\theta)\), respectively, with
diagonals fixed by row sums, e.g., \(\Rhat^\theta_t(j,j)=-\sum_{i\neq j}\Rhat^\theta_t(j,i)\).

The models below use four neural parameterizations of off-diagonal reverse rates: SEDD predicts
concrete-score entries \(s_\theta(j,t)_i\) for candidate jumps
\(j\to i\)~\citep{lou2024discrete}; MDLM and GIDD predict a distribution on clean tokens
\(x_\theta(\cdot\mid j,t)\)~\citep{sahoo2024simple,shi2024simplified,vonrutte2025gidd}; M2S predicts a posterior distribution
\(\mu_\theta(\cdot\mid j,t)\)~\citep{m2s2026}, which the operator \(B_{j,t}\) maps to
the concrete-score vector \(B_{j,t}\mu_\theta\), with
\(\displaystyle (B_{j,t}\mu_\theta)_i:=\sum_{z\in\mathcal V}\mu_\theta(z\mid j,t)
q_{t|0}(i\mid z)/q_{t|0}(j\mid z)\) for \(i\neq j\) when \(j\) is reachable at time \(t\) from
every \(z\in\mathcal V\), meaning that
\(q_{t|0}(j\mid z)>0\) for all \(z\in\mathcal V\).
Neural CTMC predicts an exit rate
\(\lambda_\theta(j,t)\) and a conditional jump distribution
\(r_\theta(\cdot\mid j,t)\)~\citep{neuralctmc2026}. Their reverse-rate maps and explicit conversions from \(x_\theta\) are given in
\Cref{thm:five-faces,thm:conversion}, while the published losses are collected in
\Cref{app:losses}.

\section{Methodology}
\label{sec:methodology}

\subsection{A Unified Reverse-Rate Theorem}
\label{sec:theory}

This subsection states the main theory, with every proof deferred to \Cref{app:proofs} and numerical
verification deferred to \Cref{app:verification}. The key message is that the losses of different discrete
diffusion models, despite their different parameterizations, can all be converted by identities into the
same Bregman divergence. We first
isolate the rate-space functional that all later parameterizations will share.

\begin{definition}
\label{def:master}
For any off-diagonal reverse-rate parameterization $\Rhat^\theta$, the \emph{pointwise rate
divergence} and the \emph{master objective} are
\begin{equation}
\label{eq:master}
  \ratekl_t(x_t,x_0;\theta):=\sum_{y\neq x_t}f\bigl(\Rhat_t(x_t,y\mid x_0),\Rhat^\theta_t(x_t,y)\bigr),
  \qquad
  \master(\theta)=\int_0^T\E_{x_0,\,x_t\sim q_{t|0}}[\ratekl_t]\,\dd t.
\end{equation}
Here $f(r,c)=D_\phi(r,c)$ is the Bregman divergence generated by
$\phi(u)=u\log u$:
\[
  D_\phi(r,c)=\phi(r)-\phi(c)-\phi'(c)(r-c)
  =r\log\tfrac{r}{c}-r+c\ge0,\qquad r,c>0.
\]
For boundary cases, we use \(f(0,c)=c\), \(f(0,0)=0\), and \(f(r,0)=+\infty\) for \(r>0\).
All time integrals are over \(0<t<T\).
\end{definition}

The next proposition connects \(\master\) to likelihood training: \(\master\) upper-bounds the negative
log-likelihood up to a constant.

\begin{proposition}
\label{prop:master}
For every \(x_0\) with \(\pdata(x_0)>0\), assume that
\begin{enumerate}
\renewcommand{\labelenumi}{\textup{(\arabic{enumi})}}
\setlength{\itemsep}{2pt}
\setlength{\parsep}{0pt}
\item \(p_{\mathrm{prior}}(k)>0\) whenever \(q_{T|0}(k\mid x_0)>0\), for every \(k\);
\item for every \(i\neq j\) and almost every \(t\in(0,T)\),
\(\Rhat^\theta_t(j,i)>0\) whenever
\(q_{t|0}(j\mid x_0)>0\) and \(\Rhat_t(j,i\mid x_0)>0\);
\item the reverse model starts from \(X_T\sim p_{\mathrm{prior}}\) and has marginal
\(\rho^\theta_t\); each entry of
\(\Rhat_t(\cdot,\cdot\mid x_0)\) and \(\Rhat^\theta_t\) is integrable on every
\([\delta,T-\varepsilon]\subset(0,T)\), with the endpoint limits
\(\lim_{t\downarrow0}\KL(q_{t|0}(\cdot\mid x_0)\|
\rho^\theta_t)=-\log p_\theta(x_0)\) and
\(\lim_{t\uparrow T}\KL(q_{t|0}(\cdot\mid x_0)\|\rho^\theta_t)
=\KL(q_{T|0}(\cdot\mid x_0)\|p_{\mathrm{prior}})\).
\end{enumerate}
Then
\[
  \E_{\pdata}[-\log p_\theta(X_0)]
  \le
  \master(\theta)+C_{\mathrm{ELBO}},
  \qquad
  C_{\mathrm{ELBO}}
  :=
  \E_{x_0\sim\pdata}
  \KL\!\left(
    q_{T|0}(\cdot\mid x_0)
    \,\middle\|\,
    p_{\mathrm{prior}}
  \right).
\]

The master objective has the following conditional and marginal rate representations:
\begin{align}
  \master(\theta)
  &=
  \int_0^T
  \sum_j q_t(j)
  \sum_{i\neq j}
  \E\!\left[
    f\bigl(
      \Rhat_t(j,i\mid X_0),
      \Rhat^\theta_t(j,i)
    \bigr)
    \,\middle|\,
    X_t=j
  \right]
  \dd t,
  \label{eq:master-conditional}\\
  &=
  \int_0^T
  \sum_j q_t(j)
  \sum_{i\neq j}
  f\bigl(
    \Rhat_t(j,i),
    \Rhat^\theta_t(j,i)
  \bigr)
  \dd t
  +C_{\mathrm{gap}}
  \label{eq:master-marginal}
\end{align}
Here $C_{\mathrm{gap}}\ge0$ is independent of \(\theta\).
\end{proposition}
Since \(C_{\mathrm{gap}}\) is independent of \(\theta\), the conditional and marginal objectives have the
same gradients whenever they are finite and differentiable. We next show that the SEDD, MDLM/GIDD,
M2S, and Neural CTMC conditional losses all reduce to \(\master\) when written in terms of their induced
reverse rates.

\begin{theorem}
\label{thm:five-faces}
The prediction heads introduced in \Cref{sec:prelim} induce off-diagonal model reverse rates through the
following mappings:
\begin{center}
\footnotesize
\setlength{\tabcolsep}{3pt}
\begin{tabular*}{\linewidth}{@{\extracolsep{\fill}}lll@{}}
\toprule
Method & Parameterization & Relation\\
\midrule
SEDD~\citep{lou2024discrete}
  & Concrete score \(s_\theta\)
  & \(\Rhat^\theta_t(j,i)=\Rfwd_t(i,j)s_\theta(j,t)_i\)\\
MDLM/GIDD~\citep{sahoo2024simple,shi2024simplified,vonrutte2025gidd}
  & Clean-token prediction \(x_\theta\)
  & \(\displaystyle \Rhat_t(j,i\mid x_\theta)
  =\Rfwd_t(i,j)\frac{q_t(i\mid x_\theta)}{q_t(j\mid x_\theta)}\)\\
M2S~\citep{m2s2026}
  & Mean-to-score \(\mu_\theta\)
  & \(\Rhat^\theta_t(j,i)=\Rfwd_t(i,j)(B_{j,t}\mu_\theta)_i\)\\
Neural CTMC~\citep{neuralctmc2026}
  & Exit rate / jump \((\lambda_\theta,r_\theta)\)
  & \(\Rhat^\theta_t(j,i)=\lambda_\theta(j,t)r_\theta(i\mid j,t)\)\\
\bottomrule
\end{tabular*}
\end{center}
For the clean-token row with \(\Rfwd_t(i,x_t)>0\) for some \(i\neq x_t\), assume
\(q_t(x_t\mid x_\theta)>0\). For \(q_{t|0}(x_t\mid x_0)>0\), substituting any row into
\Cref{def:master} gives the corresponding published
rate-space loss exactly. Thus these losses are the master objective written in different reverse-rate
coordinates.
\end{theorem}

The following corollary summarizes their equivalence in reverse-rate space.

\begin{corollary}
\label{cor:switching}
Under the same forward CTMC and noise schedule, expressing each parameterization through its induced
off-diagonal model reverse rate \(\Rhat^\theta\) yields the following equivalent conditional losses:
\[
  L_{\mathrm{SEDD}}(\Rhat^\theta)
  =L_{\mathrm{GIDD}}(\Rhat^\theta)
  =L_{\mathrm{M2S}}(\Rhat^\theta)
  =L_{\mathrm{NCTMC}}(\Rhat^\theta)
  =\master(\Rhat^\theta).
\]
Consequently, their gradients with respect to the reverse rates are identical. Over unconstrained reverse
rates, all four losses are minimized at
\(\Rhat^\theta_t(j,i)=\Rhat_t(j,i)\) on states with \(q_t(j)>0\).
\end{corollary}

The AR model provides an \(x_0\) prediction. To use it with the other three parameterizations, the
next theorem converts this prediction into the corresponding SEDD, M2S, and Neural CTMC coordinates.

\begin{theorem}
\label{thm:conversion}
Using the notation of Section~2, let \(x_\theta(\cdot\mid j,t)\in\Delta(\mathcal V)\) be an
\(x_0\) prediction for a state \(j\) that is reachable at time \(t\) from every
\(z\in\mathcal V\). Define
\[
  q_t(k\mid x_\theta):=\sum_{z\in\mathcal V}x_\theta(z)q_{t|0}(k\mid z),\qquad
  \Rhat_t(j,i\mid x_\theta):=
  \Rfwd_t(i,j)\frac{q_t(i\mid x_\theta)}{q_t(j\mid x_\theta)}\qquad (i\neq j).
\]
For a nonzero induced reverse-rate row, the corresponding SEDD, M2S, and Neural CTMC outputs are
\[
\begin{aligned}
  s_\theta(j,t)_i
  &=\frac{q_t(i\mid x_\theta)}{q_t(j\mid x_\theta)},\\
  \mu_\theta(z)
  &=\frac{x_\theta(z)q_{t|0}(j\mid z)}{q_t(j\mid x_\theta)},
  & (B_{j,t}\mu_\theta)_i
  &=\frac{q_t(i\mid x_\theta)}{q_t(j\mid x_\theta)},\\
  \lambda_\theta(j,t)
  &=\sum_{k\neq j}\Rhat_t(j,k\mid x_\theta),
  & r_\theta(i\mid j,t)
  &=\frac{\Rhat_t(j,i\mid x_\theta)}{\lambda_\theta(j,t)} .
\end{aligned}
\]
\end{theorem}

\subsection{Method}
\label{sec:method}

\ours{} adapts a pretrained AR transformer to uniform-noise diffusion through the following
techniques.

\enlargethispage{2\baselineskip}
\paragraph{AR as the \(x_0\) model.}
After a one-position shift, an AR checkpoint's next-token distribution for \(x^i\) given \(x^{<i}\)
can be used either as the clean-token prediction \(x_\theta^i(\cdot\mid x_t,t)\) or as the M2S
prediction \(\mu_\theta\).

\noindent
\begin{minipage}[t]{0.57\linewidth}
\vspace{0pt}
\Cref{fig:ar-x0-loss} compares these two choices. We initialize two models from the same AR checkpoint
(GPT2-small) and train them with the same uniform-kernel M2S loss \eqref{eq:loss-m2s}, differing only
in how the shifted AR output is used. The \(x_0\) model has \(48.6\%\) lower average loss
during the first \(2\)k updates and \(27.9\%\) lower average loss during updates \(5\)k--\(10\)k,
before the loss of the \(\mu\) model catches up after roughly \(10\)k updates. This comparison indicates
that the pretrained AR output already behaves as an \(x_0\) prediction, while using it directly as
\(\mu_\theta\) requires the model to adapt during training. We therefore use the shifted AR output as
\(x_\theta\) and convert it to SEDD, M2S, or Neural CTMC using \Cref{thm:conversion}.
\end{minipage}\hfill
\begin{minipage}[t]{0.40\linewidth}
\vspace{0pt}
\centering
\includegraphics[width=0.95\linewidth]{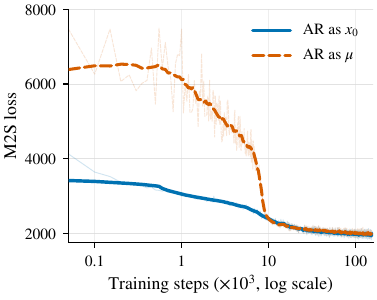}
\captionsetup{font=small,justification=raggedright,singlelinecheck=false,hypcap=false}
    \captionof{figure}{\textbf{M2S loss comparison under different parameterizations.}}
\label{fig:ar-x0-loss}
\end{minipage}
\par\medskip

\paragraph{Causal-to-bidirectional annealing.}
AR pretraining uses causal attention, whereas diffusion denoising requires bidirectional attention over
visible corrupted tokens. We anneal the attention mask from causal to fully bidirectional, so optimization
starts from the AR conditioning pattern and gradually opens the full context. Because the AR checkpoint
does not include diffusion-time conditioning, we learn the \(t\)-conditioning parameters during
adaptation while retaining the AR initialization for the remaining model parameters.

\paragraph{Uniform corruption with a linear schedule.}
We train with the uniform prior \(\pi=\1/S\). After rescaling time to \(t\in(0,1)\), the cumulative
forward corruption kernel is \(q_{t\mid0}(j\mid i)=\alpha_t\delta_{ij}+\beta_t\frac{1}{S}\), where
\(\alpha_t=1-t\) and \(\beta_t=t\).
The uniform kernel has two advantages over an absorbing mask kernel: corrupted tokens remain editable
throughout sampling, and generation retains higher lexical diversity.

\section{Experiments}
\label{sec:exp}

\subsection{Trainability of Rate-Equivalent Losses}
\label{sec:exp-loss}

We test the finite-budget trainability of four rate-equivalent formulations initialized from the same
GPT2-small AR checkpoint. All runs used uniform corruption, the same FineWeb split, and the same
\(60\)B-token training window. GIDD directly used the shifted AR softmax as its clean-token \(x_0\)
prediction. For SEDD, M2S, and Neural CTMC, we converted the same \(x_0\) prediction using
\Cref{thm:conversion} and optimized each method's native loss. Each checkpoint was evaluated with its
native \(256\)-step sampler. Training and evaluation details are provided in
\Cref{app:exp-loss-details}. \Cref{fig:method} shows the GenPPL trajectories, and \Cref{tab:loss} reports
the minimum GenPPL within the shared training window.

\begin{figure}[htbp]
\centering
\begin{minipage}[t]{0.55\linewidth}
  \vspace{0pt}
  \centering
  \includegraphics[width=\linewidth,height=0.45\linewidth]{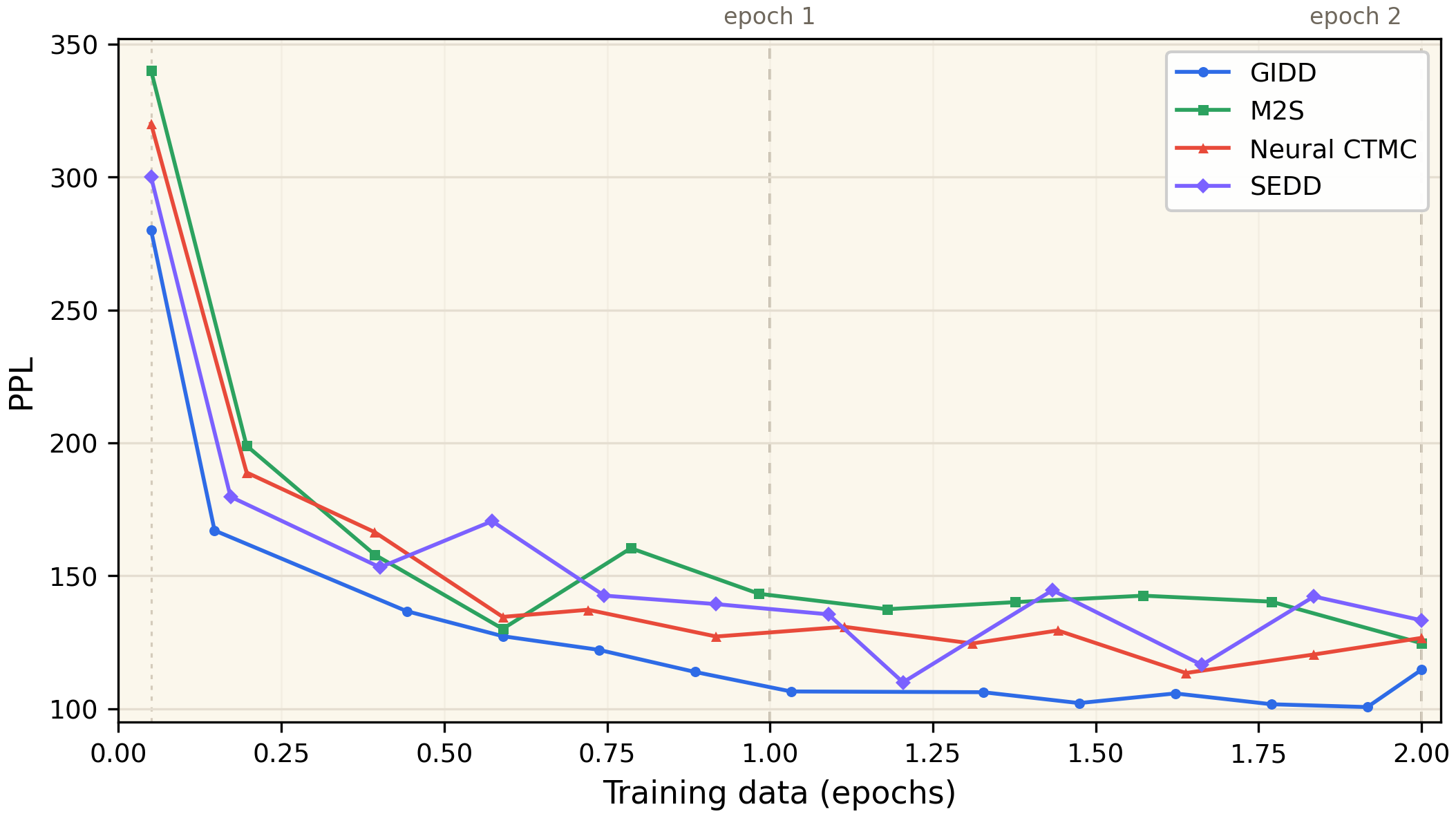}
  \caption{\textbf{Finite-budget training trajectories.} GenPPL trajectories for the four
  rate-equivalent objectives.}
  \label{fig:method}
\end{minipage}\hfill
\begin{minipage}[t]{0.41\linewidth}
  \vspace{0pt}
  \centering
  \captionsetup{type=table,justification=raggedright,singlelinecheck=false}
  \caption{\textbf{Minimum GenPPL by training objective.} Minimum GenPPL and corresponding epoch within
  the shared training window.}
  \label{tab:loss}
  \scriptsize
  \setlength{\tabcolsep}{3pt}
  \renewcommand{\arraystretch}{1.22}
  \resizebox{\linewidth}{!}{%
  \begin{tabular}{llc}
  \toprule
Objective & Epoch & Min. GenPPL $\downarrow$\\
  \midrule
  GIDD~\citep{vonrutte2025gidd}      & 1.92 & \best{100.7}\\
  SEDD~\citep{lou2024discrete}       & 1.20 & 110.0\\
  Neural CTMC~\citep{neuralctmc2026} & 1.64 & 113.5\\
  M2S~\citep{m2s2026}                & 2.00 & 124.7\\
  \bottomrule
  \end{tabular}}
\end{minipage}
\end{figure}

Rate-space equivalence (\Cref{cor:switching}) means that the four conversions produce the same reverse
rates from the same logits, which we verify within \(10^{-11}\) in \Cref{app:verification}. Their native
training runs can still follow different trajectories because their loss weighting and batch
sizes differ.
GIDD uses dynamic weighting and clips loss weights at \(2.0\), limiting the extreme endpoint weights
identified in GIDD~\citep{vonrutte2025gidd}. GIDD reached a minimum GenPPL of \(100.7\) at epoch \(1.92\);
SEDD, Neural CTMC, and M2S reached \(110.0\), \(113.5\), and \(124.7\) at epochs \(1.20\), \(1.64\), and
\(2.00\), respectively. We therefore use GIDD's native loss and sampler in the remaining experiments.

\subsection{Transfer via \texorpdfstring{\(x_0\)}{x0} Prediction}
\label{sec:exp-shift}

We next test whether the clean-token \(x_0\) output can transfer a pretrained AR model to masked and
uniform diffusion. After a one-position shift, the AR next-token prediction is used as the \(x_0\)
prediction, and training then adapts it to the target corruption kernel. We evaluate four conversion paths:
AR${\to}$mask, AR${\to}$uniform, AR${\to}$mask${\to}$uniform, and AR${\to}$uniform${\to}$mask.

\begin{wrapfigure}{r}{0.60\textwidth}
\vspace{-0.6\baselineskip}
\centering
\includegraphics[width=\linewidth]{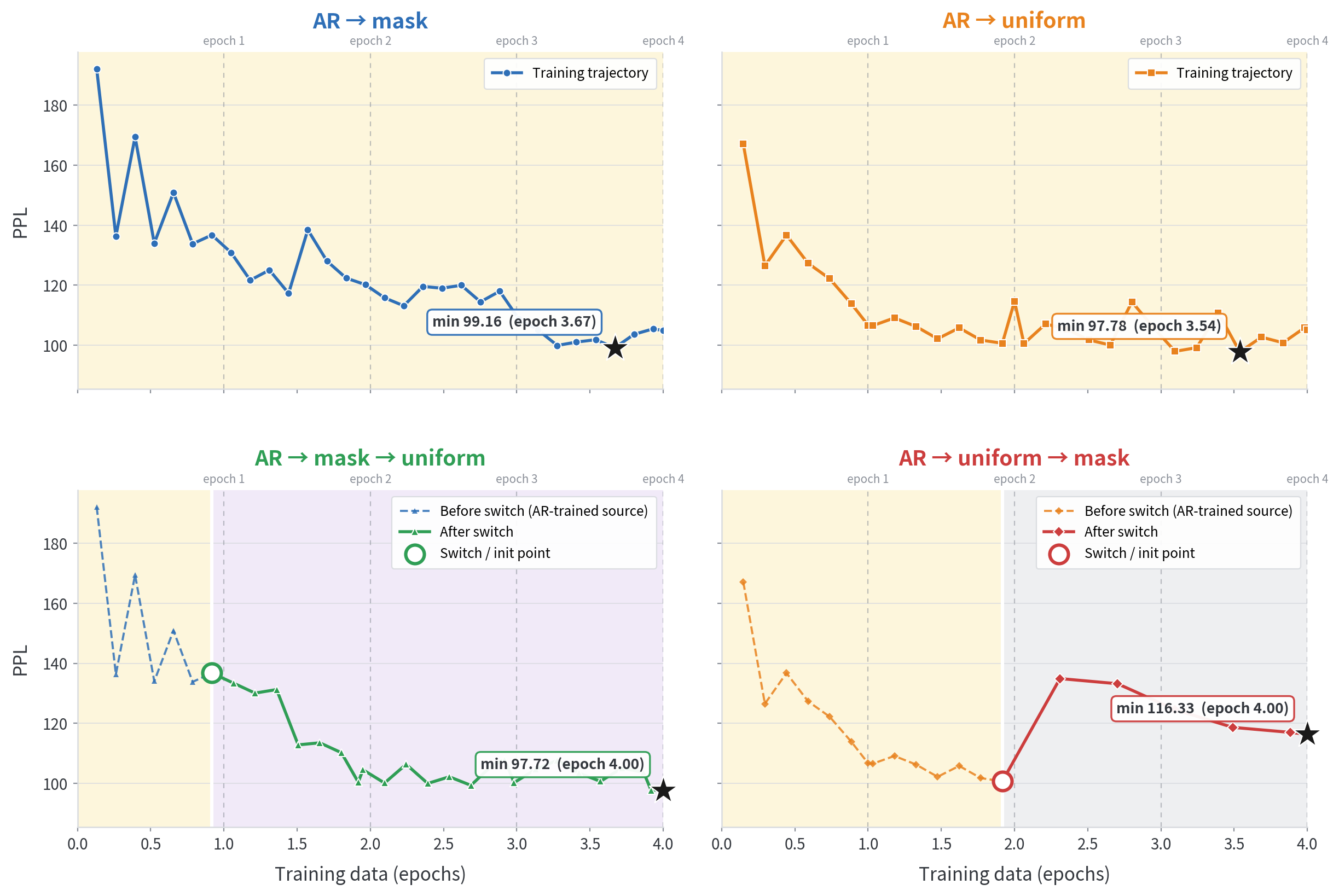}
\captionsetup{font=small}
\caption{\textbf{Conversion paths.}}
\label{fig:shift}
\vspace{-0.5\baselineskip}
\end{wrapfigure}

All four paths started from the same GPT2-small AR checkpoint and were optimized with GIDD under the
same data budget and optimization settings. Full training and evaluation details are provided in
\Cref{app:exp-shift-details}. \Cref{fig:shift} shows the GenPPL trajectory for each path and marks its
minimum value.

The direct and two-stage AR-to-uniform paths attained comparable minimum GenPPL. This indicates that
uniform diffusion can inherit the AR initialization without an intermediate masked-diffusion stage. In
contrast, uniform-to-mask conversion reached a substantially higher minimum GenPPL of \(116.33\) under
the same training budget. We therefore adopt direct AR-to-uniform conversion in subsequent experiments.

\subsection{Perplexity--Entropy Frontier}
\label{sec:exp-frontier}

GenPPL and unigram entropy are complementary but individually insufficient diagnostics: low GenPPL may
favor repetitive, high-probability text, whereas high entropy may reflect noise or random symbols rather
than meaningful diversity. We therefore analyze their Pareto frontier toward low GenPPL and high
entropy. The complete evaluation protocol and five-seed statistics are reported in
Appendix~\ref{app:frontier-setup}.

\begin{figure}[H]
\centering
\begin{minipage}[t]{0.64\linewidth}
\vspace{0pt}
\centering
\begin{tikzpicture}
\node[inner sep=0] (frontierplot) {\includegraphics[width=\linewidth]{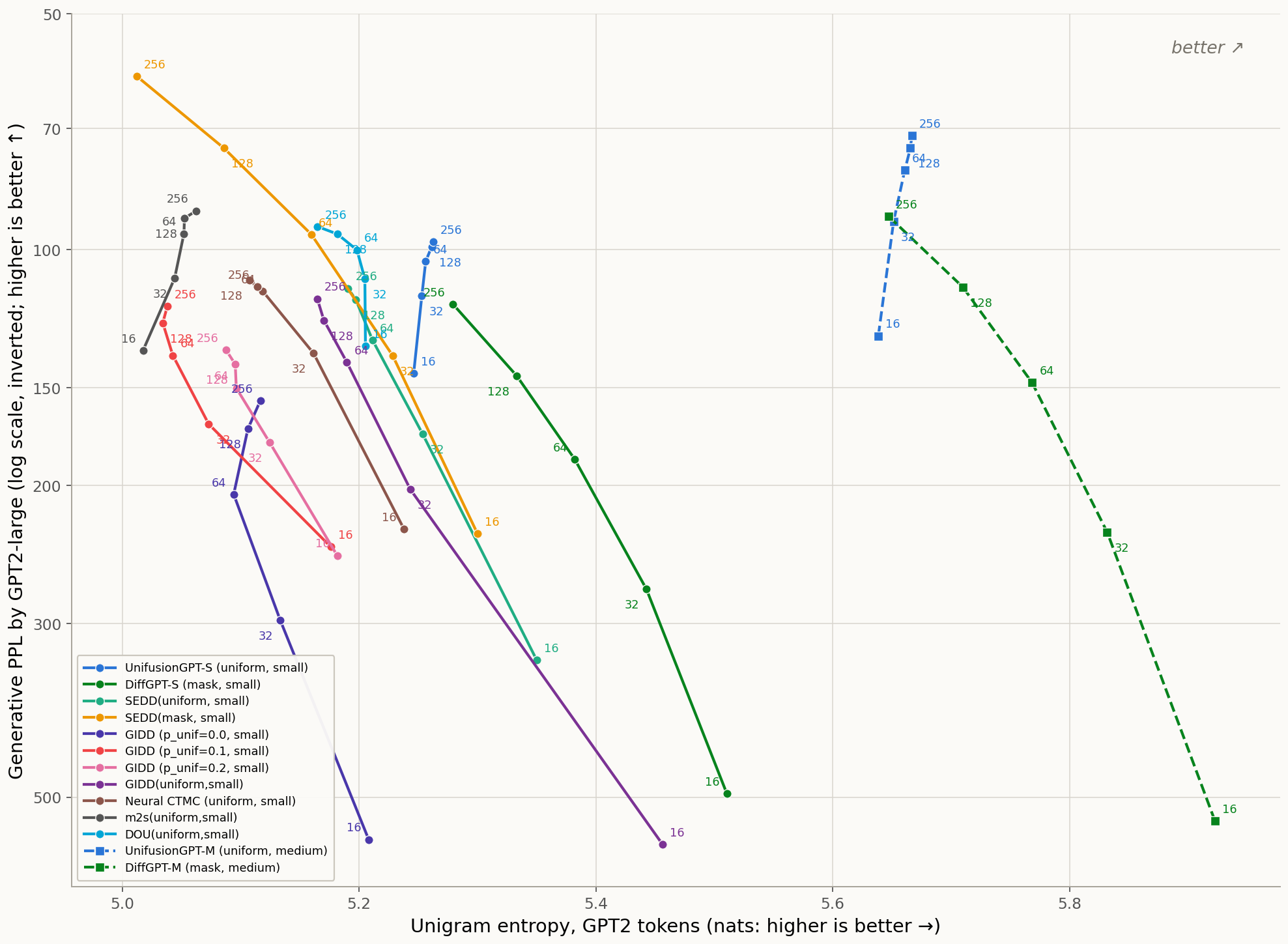}};
\begin{scope}[
  x={($(frontierplot.south east)-(frontierplot.south west)$)},
  y={($(frontierplot.north west)-(frontierplot.south west)$)},
  shift={(frontierplot.south west)}
]
\draw[-{Stealth[length=2.2mm,width=1.6mm]},draw=black!85,line width=0.8pt]
  (0.962,0.054) -- (0.998,0.054);
\draw[-{Stealth[length=2.2mm,width=1.6mm]},draw=black!85,line width=0.8pt]
  (0.056,0.940) -- (0.056,0.998);
\fill[frontierbg] (0.838,0.892) rectangle (0.992,0.982);
\node[anchor=west,inner sep=0,text=black!90,font=\bfseries\itshape\fontsize{12}{13}\selectfont]
  at (0.846,0.943) {better};
\draw[-{Stealth[length=3.4mm,width=2.5mm]},draw=black!90,line width=1.15pt,
  line cap=round,line join=round]
  (0.945,0.930) -- (0.987,0.977);
\fill[frontierbg] (0.000,0.150) rectangle (0.025,0.890);
\node[rotate=90,anchor=center,inner sep=0,text=black,
  font=\sffamily\fontsize{4.3}{5}\selectfont]
  at (0.0125,0.520)
  {Generative PPL by GPT2-large (log scale, inverted; lower is better \(\uparrow\))};
\end{scope}
\end{tikzpicture}
\end{minipage}\hfill
\begin{minipage}[t]{0.34\linewidth}
\vspace{0pt}
\centering
\begingroup
\setlength{\fboxsep}{4pt}
\fcolorbox{blue!55!black}{blue!3}{%
\begin{minipage}{\dimexpr\linewidth-2\fboxsep-2\fboxrule\relax}
\raggedright\scriptsize
\textbf{\ours{}-S, 256 steps}\\[-1pt]
\(\mathrm{GenPPL}=97.783,\ H=5.2626\)\\
\textit{Balanced frontier point.}\\
\textit{Sample:} A Center of Food and Agricultural Development provides training to agricultural industry professionals and consultants. With a degree in consulting and curriculum-based marketing\ldots
\end{minipage}}\par\nointerlineskip
\fcolorbox{green!45!black}{green!3}{%
\begin{minipage}{\dimexpr\linewidth-2\fboxsep-2\fboxrule\relax}
\raggedright\scriptsize
\textbf{DiffuGPT-S, 16 steps}\\[-1pt]
\(\mathrm{GenPPL}=494.491,\ H=5.5106\)\\
\textit{High entropy, noise-like text.}\\
\textit{Sample:} M0742 072 4TFD Fathers and Solop Over Guardians 8671 = Captain Lewis \& One Elf Piece 8670 = Hero Action Only the Queen\ldots
\end{minipage}}\par\nointerlineskip
\fcolorbox{red!60!black}{red!3}{%
\begin{minipage}{\dimexpr\linewidth-2\fboxsep-2\fboxrule\relax}
\raggedright\scriptsize
\textbf{SEDD-mask, 256 steps}\\[-1pt]
\(\mathrm{GenPPL}=60.045,\ H=5.0121\)\\
\textit{Low GenPPL, repetitive collapse.}\\
\textit{Sample:} 6 (4) of 3:00. Terminator September 3:59:00 C 20:00 9:00 C ((-c) CLINARY\ldots
\end{minipage}}
\endgroup
\end{minipage}
\caption{\textbf{(a) GenPPL--entropy frontier. (b) Selected samples.}}
\label{fig:frontier}
\end{figure}

Across 16--256 steps, \ours{}-S improves monotonically from
\(143.944\pm0.872\) to \(97.783\pm1.901\) GenPPL while its mean sample-level entropy rises from
\(5.2461\pm0.0022\) to \(5.2626\pm0.0041\) nats. Its 64-step point already Pareto-dominates the
256-step endpoints of SEDD-uniform, Neural CTMC, and every evaluated GIDD variant. At 256 steps,
\ours{}-S occupies a favorable knee of the GenPPL--entropy Pareto frontier. Compared with DiffGPT-S,
it reduces GenPPL from \(117.334\) to \(97.783\), a \(16.7\%\) improvement, while preserving nearly
the same entropy (\(5.2626\) versus \(5.2792\)). Conversely, the baselines achieving lower GenPPL
incur substantially larger diversity losses: the entropy decreases to \(5.1649\) for DOU, \(5.0624\)
for M2S, and \(5.0121\) for SEDD-mask. Consequently, no competing small model
simultaneously improves upon \ours{}-S in both metrics, demonstrating its favorable balance between
fluency and lexical diversity.

At medium scale, \ours{}-M improves from \(128.832\pm1.968\) GenPPL and
\(5.6386\pm0.0027\) entropy at 16 steps to a GenPPL of \(71.516\pm1.165\) and an entropy of
\(5.6669\pm0.0053\) at 256 steps. It Pareto-dominates DiffGPT-M at 256 steps, whose
corresponding values are \(90.645\pm2.195\) and \(5.6472\pm0.0036\). Moreover, the 64-step
\ours{}-M point (\(79.207/5.6609\)) already dominates the 256-step DiffGPT-M point, showing that the
improvement persists under a fourfold smaller sampling budget. Additional samples are provided in
Appendix~\ref{app:qual-small}.

\subsection{Downstream tasks}
\label{sec:exp-downstream}

We evaluate zero-shot accuracy on HellaSwag, WinoGrande, SocialIQA, PIQA, and BBH; the shared
evaluation protocol is detailed in Appendix~\ref{app:downstream-details}. As shown in
\Cref{tab:acc}, \ours{} achieves the highest accuracy on three of the five benchmarks at both scales.
\ours{}-S ranks first among all evaluated small models on WinoGrande (\(53.0\)), SIQA (\(37.8\)), and
BBH (\(32.0\)); \ours{}-M retains the same three-task lead at medium scale, reaching \(53.7\) on
WinoGrande, \(39.7\) on SIQA, and \(32.7\) on BBH. Both models also outperform their GPT2 initializations on all
three tasks.

\begin{table}[H]
\centering
\small
\caption{\textbf{Zero-shot downstream accuracy} (\%). Type: AR = autoregressive, DD = discrete
diffusion.}
\label{tab:acc}
\setlength{\tabcolsep}{2.5pt}
\resizebox{\textwidth}{!}{%
\begin{tabular}{llcccccccc}
\toprule
Scale & Model & Size & Type & Kernel & HSwag & Wino. & SIQA & PIQA & BBH \\
\midrule
\multirow{11}{*}{Small} & GPT2                & 124M & AR & $-$     & 29.9 & 48.5 & 35.7 & \best{62.1} & 31.3 \\
                       & DiffuGPT           & 124M & DD & mask    & \best{33.4} & 50.8 & 37.0 & \best{57.7} & 31.2 \\
                       & SEDD                & 170M & DD & mask    & 30.2 & 50.1 & 34.4 & 55.6 & 30.6 \\
                       & SEDD               & 170M & DD & uniform & 30.4 & 48.5 & 35.4 & 54.6 & 30.9 \\
                       & GIDD               & 170M & DD & hybrid ($p_{\mathrm{unif}}{=}0.0$) & 31.7 & 51.5 & 36.0 & 56.0 & 31.0 \\
                       & GIDD               & 170M & DD & hybrid ($p_{\mathrm{unif}}{=}0.1$) & 30.6 & 51.3 & 36.7 & 54.5 & 31.1 \\
                       & GIDD               & 170M & DD & hybrid ($p_{\mathrm{unif}}{=}0.2$) & 30.6 & 48.8 & 34.0 & 53.7 & 31.4 \\
                       & GIDD               & 170M & DD & uniform & 29.4 & 49.3 & 33.2 & 50.5 & 31.5 \\
                       & Neural-CTMC        & 170M & DD & uniform & 29.4 & 48.1 & 36.9 & 53.3 & 31.9 \\
                       & M2S                & 170M & DD & uniform & 30.2 & 48.4 & 36.4 & 53.4 & 31.5 \\
                       & \textsc{UnifusionGPT} & 124M & DD & uniform & 29.6 & \best{53.0} & \best{37.8} & 57.2 & \best{32.0} \\
\midrule
\multirow{4}{*}{Medium} & GPT2                & 355M & AR & $-$     & 38.3 & 50.7 & 37.7 & \best{67.4} & 32.6 \\
                        & DiffuGPT           & 355M & DD & mask    & \best{37.2} & 52.6 & 39.0 & \best{59.6} & 31.3 \\
                        & SEDD                & 424M & DD & mask    & 31.5 & 49.0 & 35.4 & 56.1 & 32.4 \\
                        & \textsc{UnifusionGPT} & 355M & DD & uniform & 35.6 & \best{53.7} & \best{39.7} & 56.8 & \best{32.7} \\
\bottomrule
\end{tabular}
}
\end{table}

\section{Related Work}
\label{sec:related}

Discrete diffusion language models build on D3PM~\citep{austin2021structured} and its continuous-time
Markov-chain formulation~\citep{campbell2022continuous}. Subsequent work follows two related directions.
The first develops different reverse-process parameterizations: SEDD predicts concrete scores
~\citep{lou2024discrete}, MDLM and GIDD predict clean tokens
~\citep{sahoo2024simple,shi2024simplified,vonrutte2025gidd}, Neural CTMC predicts exit rates and jump
distributions~\citep{neuralctmc2026}, and M2S maps posterior means to scores~\citep{m2s2026}. The second
connects absorbing diffusion to order-agnostic AR modeling and learned generation orders
~\citep{uria2014deep,hoogeboom2022autoregressive,ou2025your,garg2025learnedorder}. DiffuGPT/DiffuLLaMA
and Dream-7B use this connection to adapt pretrained AR checkpoints to absorbing diffusion
~\citep{gong2025scaling,ye2025dream}.

\section{Conclusion}
\label{sec:conclusion}

We unified the conditional losses of SEDD, MDLM/GIDD, M2S, and Neural CTMC as the same generalized KL
objective in model reverse-rate space. We further derived explicit conversions among their
parameterizations. Using the clean-token \(x_0\) prediction, \ours{} adapts GPT2 checkpoints to uniform
diffusion. On the five-seed GenPPL--entropy frontier, \ours{}-S reduces GenPPL by \(16.7\%\) relative to
DiffGPT-S at similar entropy, while \ours{}-M Pareto-dominates DiffGPT-M at 256 steps. On zero-shot
downstream tasks, both \ours{}-S and \ours{}-M achieve the best WinoGrande, SIQA, and BBH scores among
diffusion models at their respective scales.


\appendix
\clearpage

\section*{Appendix: contents}
\addcontentsline{toc}{section}{Appendix: contents}
\begingroup
\setlength{\parindent}{0pt}
\newcommand{\atocA}[2]{\makebox[1.6em][l]{\ref{#1}}#2\dotfill\pageref{#1}\par\smallskip}
\newcommand{\atocB}[2]{\hspace{1.6em}\makebox[2.4em][l]{\ref{#1}}#2\dotfill\pageref{#1}\par\smallskip}
\atocA{app:proofs}{Proofs}
\atocB{app:losses}{The losses of discrete diffusion}
\atocB{app:mainproofs}{Proofs of Section 3.1}
\atocA{app:verification}{Numerical verification}
\atocA{app:exp-details}{Additional experimental details}
\atocB{app:exp-loss-details}{Trainability of Rate-Equivalent Losses}
\atocB{app:exp-shift-details}{Transfer via \(x_0\) Prediction}
\atocB{app:frontier-setup}{Perplexity--Entropy Frontier}
\atocB{app:downstream-details}{Downstream Evaluation}
\atocB{app:training-sampling-algorithms}{Algorithms}
\atocB{app:inference-speed}{Inference Speed}
\atocA{app:qual}{Generated samples}
\atocB{app:qual-small}{Samples across sampling steps}
\atocB{app:qual-medium}{\textsc{UnifusionGPT}-M samples}
\endgroup

\begingroup
\setcounter{theorem}{0}
\renewcommand{\thetheorem}{\thesection.\arabic{theorem}}

\section{Proofs}
\label{app:proofs}

We first collect the discrete-diffusion losses in the notation of their sources in
\Cref{app:losses} and then prove the results of Section 3.1 in \Cref{app:mainproofs}.

\subsection{The losses of discrete diffusion}
\label{app:losses}

Throughout, $x_0\sim\pdata$ is the clean token/sequence and $x_t$ (or $z_t$) the corrupted sample.

\paragraph{SEDD.}
Let $s_\theta(x,t)_y\approx p_t(y)/p_t(x)$ and $K(a):=a(\log a-1)$.
The training term~\citep{lou2024discrete} is
\begin{equation}
\label{eq:loss-sedd}
  \mathcal{L}_{\mathrm{SEDD}}(\theta)
  =\int_0^T\E_{x_t\sim p_{t|0}(\cdot\mid x_0)}
  \sum_{y\neq x_t} Q_t(x_t,y)
  \Bigl(s_\theta(x_t,t)_y-\tfrac{p_{t|0}(y\mid x_0)}{p_{t|0}(x_t\mid x_0)}\log s_\theta(x_t,t)_y
    +K\bigl(\tfrac{p_{t|0}(y\mid x_0)}{p_{t|0}(x_t\mid x_0)}\bigr)\Bigr)\dd t .
\end{equation}
SEDD's rate-matrix convention is the transpose of ours, \(Q_t=\Rfwd_t^\top\).
For the uniform kernel, the rate matrix is symmetric, so \(Q_t=\Rfwd_t\).

\paragraph{MDLM.}
Under the absorbing kernel, MDLM uses the clean-token prediction
\(\mathbf{x}_\theta\).
The training term~\citep{sahoo2024simple} is
\begin{equation}
\label{eq:loss-mdlm}
  \mathcal{L}_{\mathrm{MDLM}}(\theta)
  =\E_{q}\int_0^1 \frac{\alpha'_t}{1-\alpha_t}
  \sum_{\ell:\,\mathbf z_t^\ell=\mathbf m}
  \log\bigl\langle \mathbf x_\theta^\ell(\mathbf z_t,t),\,\mathbf x^\ell\bigr\rangle\,\dd t\ (\ge0),
\end{equation}

\paragraph{GIDD.}
GIDD predicts \(\mathbf x_\theta\in\simplex\) and applies
\(q_t(z_t\mid u)=\mathrm{Cat}(z_t;\alpha_t\mathbf u+\beta_t\boldsymbol\pi_t)\)
to both \(u=\mathbf x\) and \(u=\mathbf x_\theta\).
The training term~\citep{vonrutte2025gidd} is
\begin{equation}
\label{eq:loss-gidd}
  \mathcal{L}_{\mathrm{GIDD}}(\theta)
  =\E_{\substack{t\sim\mathrm{Unif}(0,1)\\z_t\sim q_t(\cdot\mid x)}}
  \Bigl[w_t(z_t,x)\Bigl(\KL\bigl(q_t(\cdot\mid x)\|q_t(\cdot\mid\mathbf x_\theta)\bigr)
    +\tfrac{q_t(z_t\mid x)}{q_t(z_t\mid\mathbf x_\theta)}
    -\log\tfrac{q_t(z_t\mid x)}{q_t(z_t\mid\mathbf x_\theta)}-1\Bigr)\Bigr],
\end{equation}
with
\[
  w_t(z_t,x)
  =\frac{1}{q_t(z_t\mid x)}\mathbf z_t^\top
  \left(\beta_t\boldsymbol\pi_t'
  -\frac{\alpha_t'}{\alpha_t}\boldsymbol\pi_t\right).
\]

\paragraph{Neural CTMC.}
The conditional exit rate and jump distribution are defined as
\[
  \hat\lambda_t(j\mid x_0):=\sum_{i\neq j}\Rhat_t(j,i\mid x_0),
  \qquad
  \hat r_t(i\mid j,x_0):=
  \frac{\Rhat_t(j,i\mid x_0)}{\hat\lambda_t(j\mid x_0)}.
\]
The training term~\citep{neuralctmc2026} is
\begin{equation}
\label{eq:loss-nctmc-cond}
  \mathcal{L}_{\mathrm{NCTMC}}(\theta)
  =\int_0^T\E_{x_0,x_t}\Bigl[
    \KL^{\mathrm{Poi}}\bigl(\hat\lambda_t(x_t\mid x_0)\|\lambda_\theta(x_t,t)\bigr)
    +\hat\lambda_t(x_t\mid x_0)
    \KL^{\mathrm{Cat}}\bigl(\hat r_t(\cdot\mid x_t,x_0)\|r_\theta(\cdot\mid x_t,t)\bigr)
  \Bigr]\dd t .
\end{equation}
The two KL divergences are defined as
\[
  \KL^{\mathrm{Poi}}(\lambda\|\lambda')
  :=\lambda\log\frac{\lambda}{\lambda'}-\lambda+\lambda',
  \qquad
  \KL^{\mathrm{Cat}}(r\|r')
  :=\sum_i r(i)\log\frac{r(i)}{r'(i)}.
\]

\paragraph{M2S.}
For the uniform kernel \(\alpha_t=e^{-\sigma(t)}\), M2S~\citep{m2s2026} predicts
\(\mu^i_\theta\in\simplex\) and maps it to the score
\(s_{\theta,i}=B\mu^i_\theta\). For \(y\neq x_t^i\), define
\[
  r_i(x_0,x_t,t;y)
  :=\frac{q_{t|0}(y\mid x_0^i)}{q_{t|0}(x_t^i\mid x_0^i)},
  \qquad
  w_{t,i}(x_t^i,y)
  :=\Rfwd_t^{(i)}(y,x_t^i)=\frac{\dot\sigma(t)}{S}.
\]
The score-entropy divergence is
\(h(s,r):=s-r\log s+r\log r-r\), with \(h(s,0):=s\).
The training term is
\begin{equation}
\label{eq:loss-m2s}
  \mathcal{L}_{\mathrm{M2S}}(\theta)
  =\E_{t,x_0,x_t}\Bigl[\sum_{i=1}^{L}\sum_{y\neq x_t^i} w_{t,i}(x_t^i,y)\,
    h\bigl(s_{\theta,i}(x_t,t;y),r_i(x_0,x_t,t;y)\bigr)\Bigr].
\end{equation}

\subsection{Proofs of Section 3.1}
\label{app:mainproofs}

Below, \(x_0\) is the clean state, \(j=x_t\) the corrupted state, and \(i\neq j\) the reverse jump target.

\begin{appthmframe}
\begin{lemma}
\label{lem:exp-transform}
Fix \(i\neq j\). For almost every \(t\) such that \(q_t(j)>0\),
\[
  \E\!\left[
    \Rhat_t(j,i\mid X_0)
    \,\middle|\,
    X_t=j
  \right]
  =
  \Rhat_t(j,i).
\]
\end{lemma}
\end{appthmframe}
\begin{proof}
Bayes' rule gives
\[
\begin{aligned}
  \E\!\left[\Rhat_t(j,i\mid X_0)\,\middle|\,X_t=j\right]
  &=
  \sum_{z:\,q_{t|0}(j\mid z)>0}
  \frac{\pdata(z)q_{t|0}(j\mid z)}{q_t(j)}
  \Rfwd_t(i,j)\frac{q_{t|0}(i\mid z)}{q_{t|0}(j\mid z)}\\
  &=
  \frac{\Rfwd_t(i,j)}{q_t(j)}
  \sum_{z:\,q_{t|0}(j\mid z)>0}
  \pdata(z)q_{t|0}(i\mid z).
\end{aligned}
\]
If \(q_{t|0}(j\mid z)=0\), then, for almost every \(t\),
\[
  0
  =
  \frac{\dd}{\dd t}q_{t|0}(j\mid z)
  =
  \sum_{k\neq j}q_{t|0}(k\mid z)\Rfwd_t(k,j).
\]
Since \(q_{t|0}(j\mid z)\ge0\) is absolutely continuous, its derivative is zero for almost every \(t\)
at which \(q_{t|0}(j\mid z)=0\). Every term in the sum is nonnegative, so
\[
  q_{t|0}(i\mid z)\Rfwd_t(i,j)=0.
\]
Thus the sum can be extended to all \(z\), and
\[
  \E\!\left[\Rhat_t(j,i\mid X_0)\,\middle|\,X_t=j\right]
  =
  \frac{\Rfwd_t(i,j)}{q_t(j)}
  \sum_z\pdata(z)q_{t|0}(i\mid z)
  =
  \Rfwd_t(i,j)\frac{q_t(i)}{q_t(j)}
  =
  \Rhat_t(j,i).
\]
\end{proof}

\begin{appthmbox}{\Cref*{prop:master}}
For every \(x_0\) with \(\pdata(x_0)>0\), assume that
\begin{enumerate}
\renewcommand{\labelenumi}{\textup{(\arabic{enumi})}}
\setlength{\itemsep}{2pt}
\setlength{\parsep}{0pt}
\item \(p_{\mathrm{prior}}(k)>0\) whenever \(q_{T|0}(k\mid x_0)>0\), for every \(k\);
\item for every \(i\neq j\) and almost every \(t\in(0,T)\),
\(\Rhat^\theta_t(j,i)>0\) whenever
\(q_{t|0}(j\mid x_0)>0\) and \(\Rhat_t(j,i\mid x_0)>0\);
\item the reverse model starts from \(X_T\sim p_{\mathrm{prior}}\) and has marginal
\(\rho^\theta_t\); each entry of
\(\Rhat_t(\cdot,\cdot\mid x_0)\) and \(\Rhat^\theta_t\) is integrable on every
\([\delta,T-\varepsilon]\subset(0,T)\), with the endpoint limits
\(\lim_{t\downarrow0}\KL(q_{t|0}(\cdot\mid x_0)\|
\rho^\theta_t)=-\log p_\theta(x_0)\) and
\(\lim_{t\uparrow T}\KL(q_{t|0}(\cdot\mid x_0)\|\rho^\theta_t)
=\KL(q_{T|0}(\cdot\mid x_0)\|p_{\mathrm{prior}})\).
\end{enumerate}
Then
\[
  \E_{\pdata}[-\log p_\theta(X_0)]
  \le
  \master(\theta)+C_{\mathrm{ELBO}},
  \qquad
  C_{\mathrm{ELBO}}
  :=
  \E_{x_0\sim\pdata}
  \KL\!\left(
    q_{T|0}(\cdot\mid x_0)
    \,\middle\|\,
    p_{\mathrm{prior}}
  \right).
\]

The master objective has the following conditional and marginal rate representations:
\begin{align*}
  \master(\theta)
  &=
  \int_0^T
  \sum_j q_t(j)
  \sum_{i\neq j}
  \E\!\left[
    f\bigl(
      \Rhat_t(j,i\mid X_0),
      \Rhat^\theta_t(j,i)
    \bigr)
    \,\middle|\,
    X_t=j
  \right]
  \dd t
  \tag{\ref{eq:master-conditional}}\\
  &=
  \int_0^T
  \sum_j q_t(j)
  \sum_{i\neq j}
  f\bigl(
    \Rhat_t(j,i),
    \Rhat^\theta_t(j,i)
  \bigr)
  \dd t
  +C_{\mathrm{gap}}
  \tag{\ref{eq:master-marginal}}
\end{align*}
Here $C_{\mathrm{gap}}\ge0$ is independent of \(\theta\).
\end{appthmbox}

\begin{remark}
\label{rem:terminal-time}
In \eqref{eq:master-conditional}, the sum over \(j\) is restricted to \(q_t(j)>0\).
For a nonnegative integrand singular at an endpoint, we use
\(\int_0^T g_t\dd t:=\lim_{\delta,\varepsilon\downarrow0}
\int_\delta^{T-\varepsilon}g_t\dd t\), possibly \(+\infty\).
\end{remark}

\begin{proof}
\proofcase{Step 1: likelihood bound}
Fix \(x_0\), write \(q^{x_0}_t:=q_{t|0}(\cdot\mid x_0)\), and let
\(\rho^\theta_t\) be the model marginal at time \(t\). Then
\[
  q^{x_0}_0=\delta_{x_0},\qquad q^{x_0}_T=q_{T|0}(\cdot\mid x_0),
  \qquad
  \rho^\theta_0=p_\theta,\qquad \rho^\theta_T=p_{\mathrm{prior}}.
\]
Set
\[
  H_t:=\KL\!\left(q^{x_0}_t\,\middle\|\,\rho^\theta_t\right).
\]
For \(0<t<T\), the two marginals satisfy
\[
  \frac{\dd}{\dd t}q^{x_0}_t=-q^{x_0}_t\Rhat_t(\cdot,\cdot\mid x_0),
  \qquad
  \frac{\dd}{\dd t}\rho^\theta_t=-\rho^\theta_t\Rhat^\theta_t.
\]
The terminal KL limit is finite by the prior support condition, so
\(q^{x_0}_t(i)>0\Rightarrow\rho^\theta_t(i)>0\) for \(t\) sufficiently close to \(T\).
Local integrability and the rate support condition preserve this implication backward on every
truncated interval. Hence it holds for all \(0<t<T\).
Fix \(t\), set \(q=q^{x_0}_t\), \(\rho=\rho^\theta_t\),
\[
  a_{ji}:=\Rhat_t(j,i\mid x_0),\qquad
  b_{ji}:=\Rhat^\theta_t(j,i),\qquad
  u_{ji}:=\frac{q(i)\rho(j)}{\rho(i)q(j)}.
\]
For positive entries, direct differentiation gives
\[
  -\frac{\dd}{\dd t}\KL(q\|\rho)
  =
  \sum_j q(j)\sum_{i\neq j}
  \bigl\{a_{ji}\log u_{ji}-b_{ji}u_{ji}+b_{ji}\bigr\}
  \le
  \sum_j q(j)\sum_{i\neq j}f(a_{ji},b_{ji}),
\]
because
\[
  a\log u-bu+b\le a\log\frac{a}{b}-a+b=f(a,b).
\]
The boundary cases follow from the conventions in \Cref{def:master}. Integrating over
\([\delta,T-\varepsilon]\) gives
\[
  H_\delta
  \le
  H_{T-\varepsilon}
  +
  \int_\delta^{T-\varepsilon}
  \sum_j q_{t|0}(j\mid x_0)\sum_{i\neq j}
  f\bigl(\Rhat_t(j,i\mid x_0),\Rhat^\theta_t(j,i)\bigr)\dd t .
\]
The endpoint assumptions and \Cref{rem:terminal-time} then give
\[
  -\log p_\theta(x_0)
  \le
  \KL\!\left(q_{T|0}(\cdot\mid x_0)\,\middle\|\,p_{\mathrm{prior}}\right)
  +\int_0^T\sum_j q_{t|0}(j\mid x_0)\sum_{i\neq j}
    f\bigl(\Rhat_t(j,i\mid x_0),\Rhat^\theta_t(j,i)\bigr)\dd t .
\]
Averaging over \(x_0\sim\pdata\) gives the stated likelihood bound.

Since
\[
  \Pr(X_0=x_0,X_t=j)=\pdata(x_0)q_{t|0}(j\mid x_0),
\]
we have
\begin{align}
  \master(\theta)
  &=\int_0^T\sum_{x_0}\pdata(x_0)\sum_jq_{t|0}(j\mid x_0)
    \sum_{i\neq j}
    f\bigl(\Rhat_t(j,i\mid x_0),\Rhat^\theta_t(j,i)\bigr)\dd t \notag\\
  &=\int_0^T\sum_j q_t(j)\sum_{i\neq j}
    \E\!\left[
      f\bigl(\Rhat_t(j,i\mid X_0),\Rhat^\theta_t(j,i)\bigr)
      \,\middle|\, X_t=j
    \right]\dd t ,
    \label{eq:proof-master-conditional-form}
\end{align}

\proofcase{Step 2: conditional-to-marginal Bregman identity}
Fix \((t,j,i)\) with \(q_t(j)>0\), and write
\[
  R:=\Rhat_t(j,i\mid X_0),\qquad
  \bar R:=\E[R\mid X_t=j],\qquad
  c:=\Rhat^\theta_t(j,i).
\]
By \Cref{lem:exp-transform}, \(\bar R=\Rhat_t(j,i)\). If \(c=0\), the support condition gives
\[
  \Pr(R=0\mid X_t=j)=1,\qquad
  \bar R=0,\qquad
  \E[f(R,c)\mid X_t=j]=f(\bar R,c)=0.
\]
For \(c>0\),
\[
  f(r,c)=r\log r-r-r\log c+c,
\]
and therefore
\begin{align}
  \E[f(R,c)\mid X_t=j]
  &=\E[R\log R\mid X_t=j]-\bar R-\bar R\log c+c,\\
  f(\bar R,c)
  &=\bar R\log\bar R-\bar R-\bar R\log c+c.
\end{align}
Thus
\begin{equation}
\label{eq:proof-master-gap-pointwise}
  \E[f(R,c)\mid X_t=j]
  =f(\bar R,c)+\Delta_t(j,i),
  \qquad
  \Delta_t(j,i):=\E[R\log R\mid X_t=j]-\bar R\log\bar R,
\end{equation}
where \(0\log0=0\). Jensen's inequality gives \(\Delta_t(j,i)\ge0\), and
\(\Delta_t(j,i)\) does not depend on \(c\) or \(\theta\).
Multiplying by $q_t(j)$, summing, and integrating gives
\begin{align}
  \master(\theta)
  &=\int_0^T\sum_j q_t(j)\sum_{i\neq j}
    \E\!\left[f\bigl(\Rhat_t(j,i\mid X_0),\Rhat^\theta_t(j,i)\bigr)\mid X_t=j\right]\dd t\\
  &=\int_0^T\sum_j q_t(j)\sum_{i\neq j}
    f\bigl(\Rhat_t(j,i),\Rhat^\theta_t(j,i)\bigr)\dd t+C_{\mathrm{gap}},
\end{align}
where
\[
  C_{\mathrm{gap}}
  :=\int_0^T\sum_j q_t(j)\sum_{i\neq j}\Delta_t(j,i)\dd t\ge0 .
\]
This proves \eqref{eq:master-marginal}.
\end{proof}

\begin{appthmbox}{\Cref*{thm:five-faces}}
The prediction heads introduced in \Cref{sec:prelim} induce off-diagonal model reverse rates through the
following mappings:
\begin{center}
\footnotesize
\setlength{\tabcolsep}{3pt}
\begin{tabular*}{\linewidth}{@{\extracolsep{\fill}}lll@{}}
\toprule
Method & Parameterization & Relation\\
\midrule
SEDD~\citep{lou2024discrete}
  & Concrete score \(s_\theta\)
  & \(\Rhat^\theta_t(j,i)=\Rfwd_t(i,j)s_\theta(j,t)_i\)\\
MDLM/GIDD~\citep{sahoo2024simple,shi2024simplified,vonrutte2025gidd}
  & Clean-token prediction \(x_\theta\)
  & \(\displaystyle \Rhat_t(j,i\mid x_\theta)
  =\Rfwd_t(i,j)\frac{q_t(i\mid x_\theta)}{q_t(j\mid x_\theta)}\)\\
M2S~\citep{m2s2026}
  & Mean-to-score \(\mu_\theta\)
  & \(\Rhat^\theta_t(j,i)=\Rfwd_t(i,j)(B_{j,t}\mu_\theta)_i\)\\
Neural CTMC~\citep{neuralctmc2026}
  & Exit rate / jump \((\lambda_\theta,r_\theta)\)
  & \(\Rhat^\theta_t(j,i)=\lambda_\theta(j,t)r_\theta(i\mid j,t)\)\\
\bottomrule
\end{tabular*}
\end{center}
For the clean-token row with \(\Rfwd_t(i,x_t)>0\) for some \(i\neq x_t\), assume
\(q_t(x_t\mid x_\theta)>0\). For \(q_{t|0}(x_t\mid x_0)>0\), substituting any row into
\Cref{def:master} gives the corresponding published
rate-space loss exactly. Thus these losses are the master objective written in different reverse-rate
coordinates.
\end{appthmbox}

\begin{proof}
\proofcase{Setup}\par\nobreak\noindent
Fix $(x_0,j,t)$ with \(q_{t|0}(j\mid x_0)>0\). If \(\Rfwd_t(i,j)=0\) for all \(i\neq j\), set the
induced reverse-rate row to zero. Otherwise, for the clean-token row, \(q_t(j\mid x_\theta)>0\) by
assumption and the displayed ratio is well defined.
For each reverse target \(i\neq j\), the true conditional reverse branch is
\(\Rhat_t(j,i\mid x_0)=\Rfwd_t(i,j)q_{t|0}(i\mid x_0)/q_{t|0}(j\mid x_0)\).

\proofcase{SEDD}\par\nobreak\noindent
Using \(Q_t(j,i)=\Rfwd_t(i,j)\) from \Cref{app:losses}, the corresponding term in the master objective is
\begin{align}
  f\bigl(\Rhat_t(j,i\mid x_0),\Rhat^\theta_t(j,i)\bigr)
  &=f\!\left(
    Q_t(j,i)\frac{q_{t|0}(i\mid x_0)}{q_{t|0}(j\mid x_0)},
    Q_t(j,i)s_\theta(j,t)_i
  \right)\notag\\
  &=Q_t(j,i)\tfrac{q_{t|0}(i\mid x_0)}{q_{t|0}(j\mid x_0)}
    \log
    \frac{
      Q_t(j,i)\tfrac{q_{t|0}(i\mid x_0)}{q_{t|0}(j\mid x_0)}
    }{
      Q_t(j,i)s_\theta(j,t)_i
    }
    -Q_t(j,i)\tfrac{q_{t|0}(i\mid x_0)}{q_{t|0}(j\mid x_0)}
    +Q_t(j,i)s_\theta(j,t)_i\notag\\
  &=Q_t(j,i)
  \left\{
    \frac{q_{t|0}(i\mid x_0)}{q_{t|0}(j\mid x_0)}
    \log
    \frac{
      \frac{q_{t|0}(i\mid x_0)}{q_{t|0}(j\mid x_0)}
    }{
      s_\theta(j,t)_i
    }
    -\frac{q_{t|0}(i\mid x_0)}{q_{t|0}(j\mid x_0)}
    +s_\theta(j,t)_i
  \right\}\notag\\
  &=Q_t(j,i)
  \left\{
    s_\theta(j,t)_i
    -\frac{q_{t|0}(i\mid x_0)}{q_{t|0}(j\mid x_0)}
      \log s_\theta(j,t)_i
    +K\!\left(
      \frac{q_{t|0}(i\mid x_0)}{q_{t|0}(j\mid x_0)}
    \right)
  \right\}.
\end{align}
Relabeling \(j=x_t\) and \(i=y\), summing over \(y\neq x_t\), and averaging over
\(x_t\sim q_{t|0}(\cdot\mid x_0)\) gives \eqref{eq:loss-sedd}. Terms with \(Q_t(x_t,y)=0\) contribute
zero to both \(\master\) and \(\mathcal{L}_{\mathrm{SEDD}}\).

\proofcase{GIDD}\par\nobreak\noindent
For the interpolating kernel and \(i\neq j\), the weight in \Cref{app:losses} satisfies
\[
  w_t(j,x_0)
  =
  \frac{\Rfwd_t(i,j)}{q_{t|0}(j\mid x_0)}
  =
  \frac{
    \beta_t\boldsymbol\pi_t'(j)
    -\frac{\alpha_t'}{\alpha_t}\boldsymbol\pi_t(j)
  }{
    q_{t|0}(j\mid x_0)
  }.
\]
Using \(q_t(\cdot\mid x_\theta)\) from \Cref{app:losses}, the true and model reverse rates are
\[
  \Rhat_t(j,i\mid x_0)=\Rfwd_t(i,j)\frac{q_{t|0}(i\mid x_0)}{q_{t|0}(j\mid x_0)},\qquad
  \Rhat^\theta_t(j,i)=\Rfwd_t(i,j)\frac{q_t(i\mid x_\theta)}{q_t(j\mid x_\theta)} .
\]
For \(y\neq j\), substituting these rates into \(f\) gives
\begin{align}
&f\!\left(
  \Rfwd_t(y,j)\frac{q_{t|0}(y\mid x_0)}{q_{t|0}(j\mid x_0)},
  \Rfwd_t(y,j)\frac{q_t(y\mid x_\theta)}{q_t(j\mid x_\theta)}
\right)\notag\\
&\qquad=
w_t(j,x_0)
\left[
  q_{t|0}(y\mid x_0)
  \log\frac{q_{t|0}(y\mid x_0)\,q_t(j\mid x_\theta)}
           {q_{t|0}(j\mid x_0)\,q_t(y\mid x_\theta)}
  -q_{t|0}(y\mid x_0)
  +q_{t|0}(j\mid x_0)\frac{q_t(y\mid x_\theta)}{q_t(j\mid x_\theta)}
\right].
\end{align}
Using
\(\sum_{y\neq j}q_{t|0}(y\mid x_0)=1-q_{t|0}(j\mid x_0)\) and
\(\sum_{y\neq j}q_t(y\mid x_\theta)=1-q_t(j\mid x_\theta)\), summing over \(y\neq j\) gives
\begin{equation}
  \label{eq:gidd-face}
  \ratekl_t
  =
  w_t(j,x_0)
  \left[
    \KL\!\left(q_{t|0}(\cdot\mid x_0)\,\middle\|\,q_t(\cdot\mid x_\theta)\right)
    +\frac{q_{t|0}(j\mid x_0)}{q_t(j\mid x_\theta)}
    -\log\frac{q_{t|0}(j\mid x_0)}{q_t(j\mid x_\theta)}-1
  \right].
\end{equation}
Setting \(j=z_t\) and averaging over \(t\), \(x_0\), and \(z_t\) gives \eqref{eq:loss-gidd}.

\proofcase{MDLM}\par\nobreak\noindent
Under the absorbing kernel \(\pi=e_m\), a clean token can only stay unchanged or jump to \(m\), while
\(m\) never changes. Therefore, the reverse process can jump only from the masked state \(m\); an
unmasked state has zero reverse rates and contributes no loss. For \(i\in\mathcal V\),
\[
  q_{t|0}(i\mid x_0)=\alpha_t\mathbf 1\{i=x_0\},
  \qquad
  q_{t|0}(m\mid x_0)=1-\alpha_t,
  \qquad
  \Rfwd_t(i,m)=-\frac{\alpha_t'}{\alpha_t}.
\]
Since \(x_\theta(\cdot\mid m,t)\) is a distribution on \(\mathcal V\), the same kernel gives
\(q_t(i\mid x_\theta)=\alpha_t x_\theta(i\mid m,t)\) and
\(q_t(m\mid x_\theta)=1-\alpha_t\).
Thus, at \(j=m\), the true and model reverse rates are
\[
\begin{aligned}
  \Rhat_t(m,i\mid x_0)
  &=\Rfwd_t(i,m)
    \frac{q_{t|0}(i\mid x_0)}{q_{t|0}(m\mid x_0)}
   =-\frac{\alpha_t'}{1-\alpha_t}\mathbf 1\{i=x_0\},\\
  \Rhat^\theta_t(m,i)
  &=\Rfwd_t(i,m)
    \frac{q_t(i\mid x_\theta)}{q_t(m\mid x_\theta)}
   =-\frac{\alpha_t'}{1-\alpha_t}x_\theta(i\mid m,t).
\end{aligned}
\]
Substituting these rates into \(f\) and summing over \(i\in\mathcal V\) gives
\begin{align*}
  \ratekl_t
  &=
  f\!\left(
    -\frac{\alpha_t'}{1-\alpha_t},
    -\frac{\alpha_t'}{1-\alpha_t}x_\theta(x_0\mid m,t)
  \right)
  +\sum_{\substack{i\in\mathcal V\\i\neq x_0}}
  f\!\left(
    0,
    -\frac{\alpha_t'}{1-\alpha_t}x_\theta(i\mid m,t)
  \right)\\
  &=
  -\frac{\alpha_t'}{1-\alpha_t}
  \left[
    \log\frac{1}{x_\theta(x_0\mid m,t)}
    -1+x_\theta(x_0\mid m,t)
    +\sum_{\substack{i\in\mathcal V\\i\neq x_0}}x_\theta(i\mid m,t)
  \right]\\
  &=
  -\frac{\alpha_t'}{1-\alpha_t}
  \bigl[-\log x_\theta(x_0\mid m,t)\bigr],
\end{align*}
where the last equality uses \(\sum_{i\in\mathcal V}x_\theta(i\mid m,t)=1\). For a sequence, at each
masked position \(\ell\),
\[
  x_\theta(x_0\mid m,t)
  =
  \bigl\langle \mathbf x_\theta^\ell(\mathbf z_t,t),\mathbf x^\ell\bigr\rangle .
\]
Summing over masked positions and integrating over \(t\) gives \eqref{eq:loss-mdlm}.

\proofcase{M2S}\par\nobreak\noindent
At position \(\ell\), let \(j=x_t^\ell\) be the current token and \(y\neq j\) a candidate reverse
jump. M2S defines
\begin{align*}
  s_{\theta,\ell}(x_t,t;y)
  &=(B_{j,t}\mu_\theta^\ell)_y
    =\sum_{z\in\mathcal V}[\mu_\theta^\ell(x_t,t)]_z
      \frac{q_{t|0}(y\mid z)}{q_{t|0}(j\mid z)},\\
  r_\ell(x_0,x_t,t;y)
  &=\frac{q_{t|0}(y\mid x_0^\ell)}{q_{t|0}(j\mid x_0^\ell)},\qquad
  w_{t,\ell}(j,y)=\Rfwd_t^{(\ell)}(y,j).
\end{align*}
The true and model reverse rates for the jump \(j\to y\) are therefore
\[
  \Rhat_t(j,y\mid x_0)=w_{t,\ell}(j,y)r_\ell(x_0,x_t,t;y),\qquad
  \Rhat^\theta_t(j,y)=w_{t,\ell}(j,y)s_{\theta,\ell}(x_t,t;y).
\]
Substituting these rates into \(f(r,c)=r\log(r/c)-r+c\) gives
\begin{align*}
  f\bigl(\Rhat_t(j,y\mid x_0),\Rhat^\theta_t(j,y)\bigr)
  &=
  w_{t,\ell}(j,y)
  \left[
    s_{\theta,\ell}
    -r_\ell\log s_{\theta,\ell}
    +r_\ell\log r_\ell-r_\ell
  \right]\\
  &=w_{t,\ell}(j,y)\,
    h\bigl(s_{\theta,\ell}(x_t,t;y),r_\ell(x_0,x_t,t;y)\bigr).
\end{align*}
Summing over \(y\neq x_t^\ell\) and all positions \(\ell\), then averaging over
\((t,x_0,x_t)\), gives \eqref{eq:loss-m2s}.

\proofcase{Neural CTMC}\par\nobreak\noindent
At the current state \(j\), Neural CTMC predicts a total jump rate \(\lambda_\theta(j,t)\) and
probabilities \(r_\theta(i\mid j,t)\) for the targets \(i\neq j\). Thus
\[
  \Rhat^\theta_t(j,i)=\lambda_\theta(j,t)r_\theta(i\mid j,t),\qquad
  \sum_{i\neq j}r_\theta(i\mid j,t)=1 .
\]
The true total jump rate and the corresponding jump probabilities are
\[
  \hat\lambda_t(j\mid x_0)
  =
  \sum_{i\neq j}\Rhat_t(j,i\mid x_0)
  =
  \sum_{i\neq j}\Rfwd_t(i,j)\frac{q_{t|0}(i\mid x_0)}{q_{t|0}(j\mid x_0)},
  \qquad
  \hat r_t(i\mid j,x_0)
  =
  \frac{\Rhat_t(j,i\mid x_0)}{\hat\lambda_t(j\mid x_0)} .
\]
Therefore \(\Rhat_t(j,i\mid x_0)=\hat\lambda_t(j\mid x_0)\hat r_t(i\mid j,x_0)\) and
\(\sum_{i\neq j}\hat r_t(i\mid j,x_0)=1\).
Substituting the true and model rates into \(f\) and using that both jump distributions sum to one gives
\begin{align*}
  \ratekl_t
  &=\sum_{i\neq j}
    f\!\left(
      \hat\lambda_t(j\mid x_0)\hat r_t(i\mid j,x_0),
      \lambda_\theta(j,t)r_\theta(i\mid j,t)
    \right)\\
  &=\hat\lambda_t(j\mid x_0)
      \log\frac{\hat\lambda_t(j\mid x_0)}{\lambda_\theta(j,t)}
    -\hat\lambda_t(j\mid x_0)
    +\lambda_\theta(j,t)
    +\hat\lambda_t(j\mid x_0)
      \sum_{i\neq j}\hat r_t(i\mid j,x_0)
      \log\frac{\hat r_t(i\mid j,x_0)}{r_\theta(i\mid j,t)}\\
  &=\KL^{\mathrm{Poi}}\!\left(
      \hat\lambda_t(j\mid x_0)\middle\|\lambda_\theta(j,t)
    \right)
    +\hat\lambda_t(j\mid x_0)
    \KL^{\mathrm{Cat}}\!\left(
      \hat r_t(\cdot\mid j,x_0)\middle\|r_\theta(\cdot\mid j,t)
    \right).
\end{align*}
If \(\hat\lambda_t(j\mid x_0)=0\), every true edge rate is zero and
\[
  \ratekl_t
  =\sum_{i\neq j}f\bigl(0,\lambda_\theta(j,t)r_\theta(i\mid j,t)\bigr)
  =\lambda_\theta(j,t)
  =\KL^{\mathrm{Poi}}\!\bigl(0\|\lambda_\theta(j,t)\bigr);
\]
the categorical term is zero because it is multiplied by \(\hat\lambda_t(j\mid x_0)\).
Averaging over \((X_0,X_t,t)\) gives \eqref{eq:loss-nctmc-cond}.
\end{proof}

\begin{appthmbox}{\Cref*{cor:switching}}
Under the same forward CTMC and noise schedule, expressing each parameterization through its induced
off-diagonal model reverse rate \(\Rhat^\theta\) yields the following equivalent conditional losses:
\[
  L_{\mathrm{SEDD}}(\Rhat^\theta)
  =L_{\mathrm{GIDD}}(\Rhat^\theta)
  =L_{\mathrm{M2S}}(\Rhat^\theta)
  =L_{\mathrm{NCTMC}}(\Rhat^\theta)
  =\master(\Rhat^\theta).
\]
Consequently, their gradients with respect to the reverse rates are identical. Over unconstrained reverse
rates, all four losses are minimized at
\(\Rhat^\theta_t(j,i)=\Rhat_t(j,i)\) on states with \(q_t(j)>0\).
\end{appthmbox}

\begin{proof}
By \Cref{thm:five-faces}, substituting any of these heads into the conditional objective
\Cref{def:master} gives the same functional of the induced off-diagonal reverse rate:
\[
  L_{\mathrm{SEDD}}(\Rhat^\theta)
  =L_{\mathrm{GIDD}}(\Rhat^\theta)
  =L_{\mathrm{M2S}}(\Rhat^\theta)
  =L_{\mathrm{NCTMC}}(\Rhat^\theta)
  =\master(\Rhat^\theta).
\]
Thus the values are identical. For an edge with \(q_t(j)>0\) and
\(\Rhat^\theta_t(j,i)>0\), compute the functional derivative directly from the conditional form of
\(\master\):
\[
\begin{aligned}
  \frac{\delta \master}{\delta \Rhat^\theta_t(j,i)}
  &=
  q_t(j)\,
  \E\!\left[
    \partial_c f\bigl(\Rhat_t(j,i\mid X_0),\Rhat^\theta_t(j,i)\bigr)
    \,\middle|\, X_t=j
  \right] \\
  &=
  q_t(j)\,
  \E\!\left[
    1-\frac{\Rhat_t(j,i\mid X_0)}{\Rhat^\theta_t(j,i)}
    \,\middle|\, X_t=j
  \right] \\
  &=
  q_t(j)\left(1-\frac{\Rhat_t(j,i)}{\Rhat^\theta_t(j,i)}\right).
\end{aligned}
\]
The last equality uses \(\E[\Rhat_t(j,i\mid X_0)\mid X_t=j]=\Rhat_t(j,i)\)
from \Cref{lem:exp-transform}. The corresponding second variation on this edge is
\[
  \frac{\delta^2 \master}
       {\delta\bigl(\Rhat^\theta_t(j,i)\bigr)^2}
  =q_t(j)\frac{\Rhat_t(j,i)}{\bigl(\Rhat^\theta_t(j,i)\bigr)^2}.
\]
It is strictly positive whenever \(q_t(j)\Rhat_t(j,i)>0\). Thus the shared rate-space functional is
strictly convex on every active true edge, and its unique minimizer there is
\(\Rhat^\theta_t(j,i)=\Rhat_t(j,i)\) for \(i\neq j\).
If \(\Rhat_t(j,i)=0\), then
\(f\bigl(0,\Rhat^\theta_t(j,i)\bigr)=\Rhat^\theta_t(j,i)\), so the minimum is attained at
\(\Rhat^\theta_t(j,i)=0\).
\end{proof}

\begin{appthmbox}{\Cref*{thm:conversion}}
Using the notation of Section~2, let \(x_\theta(\cdot\mid j,t)\in\Delta(\mathcal V)\) be an
\(x_0\) prediction for a state \(j\) that is reachable at time \(t\) from every
\(z\in\mathcal V\). Define
\[
  q_t(k\mid x_\theta):=\sum_{z\in\mathcal V}x_\theta(z)q_{t|0}(k\mid z),\qquad
  \Rhat_t(j,i\mid x_\theta):=
  \Rfwd_t(i,j)\frac{q_t(i\mid x_\theta)}{q_t(j\mid x_\theta)}\qquad (i\neq j).
\]
For a nonzero induced reverse-rate row, the corresponding SEDD, M2S, and Neural CTMC outputs are
\[
\begin{aligned}
  s_\theta(j,t)_i
  &=\frac{q_t(i\mid x_\theta)}{q_t(j\mid x_\theta)},\\
  \mu_\theta(z)
  &=\frac{x_\theta(z)q_{t|0}(j\mid z)}{q_t(j\mid x_\theta)},
  & (B_{j,t}\mu_\theta)_i
  &=\frac{q_t(i\mid x_\theta)}{q_t(j\mid x_\theta)},\\
  \lambda_\theta(j,t)
  &=\sum_{k\neq j}\Rhat_t(j,k\mid x_\theta),
  & r_\theta(i\mid j,t)
  &=\frac{\Rhat_t(j,i\mid x_\theta)}{\lambda_\theta(j,t)} .
\end{aligned}
\]
\end{appthmbox}

\begin{proof}
Fix \(j\) and \(t\), and abbreviate \(x_\theta(z\mid j,t)\) by \(x_\theta(z)\).
Reachability gives \(q_{t|0}(j\mid z)>0\) for every \(z\), so
\(q_t(j\mid x_\theta)=\sum_zx_\theta(z)q_{t|0}(j\mid z)>0\).
Thus all kernel ratios below are well defined.

\proofcase{SEDD}\par\nobreak\noindent
SEDD represents the rate on \(j\to i\) as \(\Rfwd_t(i,j)s_\theta(j,t)_i\), whereas the
clean-token prediction gives
\(\Rhat_t(j,i\mid x_\theta)=\Rfwd_t(i,j)q_t(i\mid x_\theta)/q_t(j\mid x_\theta)\).
Matching the two expressions gives the stated choice
\[
  s_\theta(j,t)_i
  =\frac{q_t(i\mid x_\theta)}{q_t(j\mid x_\theta)}.
\]

\proofcase{M2S}\par\nobreak\noindent
To represent the clean-token-induced score with the M2S head, choose \(\mu_\theta\) so that
\((B_{j,t}\mu_\theta)_i=q_t(i\mid x_\theta)/q_t(j\mid x_\theta)\).
Expanding both sides gives
\[
  \sum_z\frac{\mu_\theta(z)}{q_{t|0}(j\mid z)}q_{t|0}(i\mid z)
  =
  \sum_z\frac{x_\theta(z)}{q_t(j\mid x_\theta)}q_{t|0}(i\mid z).
\]
Making the two sums equal term by term gives
\(\mu_\theta(z)/q_{t|0}(j\mid z)=x_\theta(z)/q_t(j\mid x_\theta)\), hence
\[
  \mu_\theta(z)
  =\frac{x_\theta(z)q_{t|0}(j\mid z)}{q_t(j\mid x_\theta)}.
\]
Its numerator sums to \(q_t(j\mid x_\theta)\), so \(\mu_\theta\) is a probability distribution.
Substituting it into \(B_{j,t}\) gives
\[
\begin{aligned}
  (B_{j,t}\mu_\theta)_i
  &=\sum_z
    \frac{x_\theta(z)q_{t|0}(j\mid z)}{q_t(j\mid x_\theta)}
    \frac{q_{t|0}(i\mid z)}{q_{t|0}(j\mid z)}\\
  &=\frac{\sum_zx_\theta(z)q_{t|0}(i\mid z)}{q_t(j\mid x_\theta)}
   =\frac{q_t(i\mid x_\theta)}{q_t(j\mid x_\theta)}.
\end{aligned}
\]
Here \(q_{t|0}(j\mid z)>0\) justifies the cancellation. Thus the M2S score equals the
SEDD score above and induces \(\Rhat_t(j,i\mid x_\theta)\).

\proofcase{Neural CTMC}\par\nobreak\noindent
Neural CTMC writes
\(\Rhat_t(j,i\mid x_\theta)=\lambda_\theta(j,t)r_\theta(i\mid j,t)\), where the jump
probabilities sum to one. Summing over all targets gives
\[
  \lambda_\theta(j,t)
  =\lambda_\theta(j,t)\sum_{k\neq j}r_\theta(k\mid j,t)
  =\sum_{k\neq j}\Rhat_t(j,k\mid x_\theta).
\]
The row is nonzero, so \(\lambda_\theta(j,t)>0\). Dividing each edge rate by this total gives
\[
  r_\theta(i\mid j,t)
  =\frac{\Rhat_t(j,i\mid x_\theta)}{\lambda_\theta(j,t)}.
\]
These probabilities are nonnegative and sum to one. Hence all three outputs induce the stated
reverse-rate row.
\end{proof}

\section{Numerical verification}
\label{app:verification}

This appendix tests whether mathematically equivalent formulas and code paths return the same
numerical result for the same input.

\begin{center}
\captionsetup{type=table,hypcap=false}
\captionof{table}{\textbf{Numerical verification.}}
\label{tab:identity-verification}
\vspace{0.25em}
\footnotesize
\setlength{\tabcolsep}{3.2pt}
\begin{tabular}{@{}
  >{\raggedright\arraybackslash}p{0.19\linewidth}
  >{\raggedright\arraybackslash}p{0.51\linewidth}
  >{\raggedright\arraybackslash}p{0.23\linewidth}
@{}}
\toprule
Check & Question answered & Result \\
\midrule
\multicolumn{3}{@{}l}{\textit{Formula-level checks}} \\
Formula values
  & Do the two sides of every identity in \Cref{app:proofs} give the same value?
  & \(<10^{-13}\) \\
Formula gradients
  & Does the analytic gradient match a finite-difference estimate?
  & \(<10^{-10}\) \\
Formula optima
  & Does the closed-form optimum match the result of numerical optimization?
  & \(<10^{-8}\) \\
\addlinespace[0.35em]
\multicolumn{3}{@{}l}{\textit{Code-level checks}} \\
Checkpoint rates
  & For the same \ours{}-S logits, do the SEDD, M2S, and Neural-CTMC rates match the GIDD rates?
  & \(<10^{-11}\) \\
Checkpoint loss
  & Does the GIDD loss match \(\mathcal D_t\)?
  & \(<10^{-6}\) \\
\bottomrule
\end{tabular}
\end{center}

\paragraph{Small-state checks.}
We directly compare both sides of every identity in \Cref{app:proofs}, analytic gradients with
finite-difference estimates, and closed-form optima with numerical optimization. Each check uses
20 random instances with vocabulary size \(S\in\{3,\dots,8\}\), random data distributions, kernels,
and model heads. For sequence identities, we set \(L=3\) and sum over all clean sequences and all
masking patterns or permutations.

\paragraph{Code checks.}
We hold a trained \ours{}-S checkpoint fixed and use the same logits to compare the reverse rates
produced by GIDD, SEDD, M2S, and Neural CTMC. We also compare the GIDD loss with
\(\mathcal D_t\). \Cref{tab:identity-verification} reports the largest observed difference for each
comparison.

\section{Additional experimental details}
\label{app:exp-details}

\subsection{Trainability of Rate-Equivalent Losses}
\label{app:exp-loss-details}

All four runs in \Cref{sec:exp-loss} start from the same GPT2-small AR checkpoint and use the
FineWeb training split with a cumulative training budget of \(60\)B tokens. We keep the model
initialization and optimization recipe fixed across objectives: AdamW with a cosine learning-rate schedule
peaking at \(3\times10^{-4}\), linear uniform noise, and \(10\)k-step causal-to-bidirectional attention
annealing. For GIDD, the shifted AR softmax directly parameterizes the clean-token \(x_0\) predictor. For
SEDD, M2S, and Neural CTMC, we map the same \(x_0\) prediction to the corresponding parameterization using
\Cref{thm:conversion} and optimize each method's native loss.

For each checkpoint, we generate \(1024\) unconditional samples of length \(512\) using \(256\) sampling
steps and compute GenPPL on the complete generated text with GPT2-large.

\subsection{Transfer via \texorpdfstring{\(x_0\)}{x0} Prediction}
\label{app:exp-shift-details}

All four conversion paths in \Cref{sec:exp-shift} start from the same GPT2-small AR checkpoint and are
optimized with GIDD. Each run uses the FineWeb training split with a cumulative training budget of
\(120\)B tokens, AdamW with a cosine learning-rate schedule peaking at \(3\times10^{-4}\), a linear noise
schedule, and \(10\)k-step causal-to-bidirectional attention annealing. Apart from the kernel transitions,
the training recipe is fixed across paths.

We evaluate every checkpoint using the protocol of \Cref{app:exp-loss-details}: \(1024\) unconditional
samples of length \(512\), \(256\) sampling steps, and GPT2-large GenPPL computed on the complete
generated text.

\subsection{Perplexity--Entropy Frontier}
\label{app:frontier-setup}

\paragraph{\textsc{UnifusionGPT}-S.}
We initialize the 124M-parameter model from the pretrained GPT2-small checkpoint and apply the
one-position shift. We train it on a 30B-token subset of FineWeb~\citep{penedo2024fineweb} with
the GIDD loss under linear uniform corruption ($\alpha_t=1-t$, $\beta_t=t$). The global batch size is
864 across 24 H100 GPUs. With a maximum sequence length of 512 and 240,000 training
steps, the model processes approximately 106.2B tokens, equivalent to 3.54 epochs. We use AdamW with
a learning rate of
$3\!\times\!10^{-4}$ with a cosine scheduler and 3,000 warmup steps. The attention mask is annealed
from causal to fully bidirectional over the first 10,000 steps. We train in bf16 with DeepSpeed ZeRO-2
parallelization.

\paragraph{\textsc{UnifusionGPT}-M.}
We initialize the 355M-parameter model from the pretrained GPT2-medium checkpoint and apply the
one-position shift. We train it on a 30B-token subset of FineWeb~\citep{penedo2024fineweb} with the
GIDD loss under linear uniform corruption ($\alpha_t=1-t$, $\beta_t=t$). The global batch size is 864
across 24 H100 GPUs. With a maximum sequence length of 1024 and 255,000 training steps, the model
processes approximately 225.6B tokens, equivalent to 7.52 epochs. We use AdamW with a learning rate of
$3\!\times\!10^{-4}$ with a cosine scheduler and 3,000 warmup steps. The attention mask is annealed
from causal to fully bidirectional over the first 10,000 steps. We train in bf16 with DeepSpeed ZeRO-2
parallelization.

\paragraph{Baseline checkpoints.}
To reflect the strongest available performance of each baseline, we use the officially released
checkpoint and accompanying sampling procedure whenever available; otherwise, we train the model
under the matched protocol described below.

We use the officially released checkpoints for \textbf{DiffGPT-S (mask, small)} and
\textbf{DiffGPT-M (mask, medium)}.
DiffGPT-S contains 124.4M parameters and uses the GPT2-small architecture
(12 layers, 12 heads, hidden size 768), with a maximum sequence length of 512.
DiffGPT-M contains 354.8M parameters and uses the GPT2-medium architecture
(24 layers, 16 heads, hidden size 1024), with a maximum sequence length of 1024. Both are adapted on
a 30B-token FineWeb subset; their cumulative training budgets are approximately 131B tokens for
DiffGPT-S and 210B tokens for DiffGPT-M. Both are sampled with the Bayesian reverse
process~\citep{gong2025scaling}.

\textbf{SEDD(mask, small)} is the official 169.6M-parameter OpenWebText checkpoint trained for
approximately 682B tokens with a maximum sequence length of 1024, while
\textbf{SEDD(uniform, small)} is our locally trained 169.6M-parameter checkpoint trained on 131B
OpenWebText tokens with a maximum sequence length of 512. To facilitate a controlled comparison, the
latter matches the 12-block, 12-head DDiT backbone (hidden size 768) and optimization hyperparameters
of \textbf{SEDD(mask, small)}, changing the forward kernel from absorbing to uniform. For both
checkpoints, we use the official score-based ancestral sampler~\citep{lou2024discrete}.

\textbf{GIDD (p\_unif=0.0, small)},
\textbf{GIDD (p\_unif=0.1, small)}, and \textbf{GIDD (p\_unif=0.2, small)} are 169.8M-parameter DiTs
(12 layers, 12 heads, hidden size 768) with a maximum sequence length of 512, trained on 131B
OpenWebText tokens, and sampled with the official ancestral denoising
method~\citep{vonrutte2025gidd}. \textbf{GIDD(uniform, small)} is trained locally using the same
configuration but with a uniform forward noising kernel, and is sampled with the official ancestral
denoising method.

\textbf{DOU(uniform, small)} is the officially released 169.6M-parameter Duo checkpoint, based on a
modified DiT with 12 layers, 12 attention heads, a hidden size of 768, and a maximum sequence length
of 1024, trained on approximately 524B OpenWebText tokens and sampled with the official ancestral
sampler~\citep{sahoo2025diffusion}.

\textbf{Neural CTMC (uniform, small)}~\citep{neuralctmc2026} and
\textbf{M2S (uniform, small)}~\citep{m2s2026} are 169.7M-parameter, 12-block, 12-head DiT models with hidden size 768
and a maximum sequence length of 512, trained on 131B OpenWebText tokens and sampled with Euler
integration.

\paragraph{Evaluation protocol.}
For each model, sampling budget in $\{16,32,64,128,256\}$, and seed in
$\{123,456,789,2024,3407\}$, we draw (1024) unconditional samples at temperature $1.0$, with
generation lengths of (512) tokens for small models and (1024) tokens for medium models. Within a seed,
GenPPL is evaluated with GPT2-large by summing shifted
next-token negative log-likelihood over valid non-padding tokens, dividing by the corresponding valid-token
count, and exponentiating. We report the arithmetic mean and sample standard deviation of the five
per-seed GenPPL values.

For entropy, each saved text is independently retokenized with the GPT2 tokenizer without special tokens.
For a sample (x), we compute its empirical unigram Shannon entropy
\[
  \widehat H(x)=-\sum_{v:c_v(x)>0}\frac{c_v(x)}{N_x}\log\frac{c_v(x)}{N_x},
  \qquad N_x=\sum_v c_v(x).
\]
The (1024) sample-level entropies are averaged to obtain one value per seed; the reported entropy is the
mean and sample standard deviation of these five seed-level averages.

\begin{table}[H]
\centering
\scriptsize
\caption{\textbf{Five-seed values displayed in \Cref{fig:frontier}.} Each entry reports mean
$\pm$ sample standard deviation across five seeds. Panel (a) reports GPT2-large GenPPL; panel (b)
reports sample-level unigram entropy, averaged within each seed before aggregation across seeds.}
\label{tab:frontier-full}
\renewcommand{\arraystretch}{1.05}
\setlength{\tabcolsep}{2.5pt}
\textbf{(a) GenPPL mean $\pm$ std}\\[3pt]
\resizebox{\textwidth}{!}{%
\begin{tabular}{llccccc}
\toprule
Model & Kernel & 16 & 32 & 64 & 128 & 256\\
\midrule
\textsc{UnifusionGPT}-S        & uniform & $143.944 \pm 0.872$ & $114.571 \pm 0.850$ & $103.441 \pm 1.394$ & $99.199 \pm 0.930$ & $97.783 \pm 1.901$\\
DiffGPT-S                      & mask    & $494.491 \pm 4.839$ & $271.210 \pm 4.260$ & $185.117 \pm 0.621$ & $144.980 \pm 2.223$ & $117.334 \pm 0.894$\\
SEDD                           & uniform & $334.007 \pm 10.592$ & $171.825 \pm 2.967$ & $130.487 \pm 2.213$ & $115.904 \pm 0.616$ & $112.059 \pm 0.664$\\
SEDD                           & mask    & $230.571 \pm 1.500$ & $136.543 \pm 1.691$ & $95.625 \pm 1.052$ & $74.196 \pm 0.624$ & $60.045 \pm 0.400$\\
GIDD ($p_{\rm unif}{=}0.0$)   & hybrid  & $566.066 \pm 10.590$ & $297.354 \pm 4.299$ & $205.251 \pm 5.606$ & $169.339 \pm 3.386$ & $155.875 \pm 3.181$\\
GIDD ($p_{\rm unif}{=}0.1$)   & hybrid  & $239.518 \pm 2.855$ & $166.977 \pm 1.344$ & $136.454 \pm 2.193$ & $124.020 \pm 2.357$ & $118.165 \pm 1.459$\\
GIDD ($p_{\rm unif}{=}0.2$)   & hybrid  & $245.696 \pm 4.107$ & $176.264 \pm 1.199$ & $150.196 \pm 2.064$ & $140.027 \pm 2.121$ & $134.277 \pm 1.256$\\
GIDD                           & uniform & $574.170 \pm 15.752$ & $202.289 \pm 4.054$ & $139.239 \pm 1.750$ & $123.097 \pm 1.535$ & $115.513 \pm 0.909$\\
Neural CTMC                    & uniform & $227.430 \pm 3.214$ & $135.640 \pm 2.144$ & $113.019 \pm 0.772$ & $109.279 \pm 0.727$ & $111.516 \pm 1.070$\\
M2S                            & uniform & $134.526 \pm 1.115$ & $108.765 \pm 0.904$ & $95.384 \pm 0.326$ & $91.197 \pm 1.308$ & $89.204 \pm 0.499$\\
DOU                            & uniform & $132.835 \pm 1.205$ & $108.972 \pm 0.844$ & $100.086 \pm 1.525$ & $95.529 \pm 1.502$ & $93.433 \pm 1.005$\\
\textsc{UnifusionGPT}-M        & uniform & $128.832 \pm 1.968$ & $92.007 \pm 1.547$ & $79.207 \pm 1.215$ & $74.162 \pm 0.901$ & $71.516 \pm 1.165$\\
DiffGPT-M                      & mask    & $535.982 \pm 7.893$ & $229.408 \pm 2.128$ & $147.867 \pm 1.495$ & $111.603 \pm 1.412$ & $90.645 \pm 2.195$\\
\bottomrule
\end{tabular}
}

\vspace{8pt}
\textbf{(b) Unigram entropy mean $\pm$ std}\\[3pt]
\resizebox{\textwidth}{!}{%
\begin{tabular}{llccccc}
\toprule
Model & Kernel & 16 & 32 & 64 & 128 & 256\\
\midrule
\textsc{UnifusionGPT}-S        & uniform & $5.2461 \pm 0.0022$ & $5.2530 \pm 0.0020$ & $5.2563 \pm 0.0027$ & $5.2613 \pm 0.0021$ & $5.2626 \pm 0.0041$\\
DiffGPT-S                      & mask    & $5.5106 \pm 0.0017$ & $5.4427 \pm 0.0023$ & $5.3821 \pm 0.0013$ & $5.3333 \pm 0.0034$ & $5.2792 \pm 0.0071$\\
SEDD                           & uniform & $5.3500 \pm 0.0111$ & $5.2536 \pm 0.0063$ & $5.2114 \pm 0.0059$ & $5.1973 \pm 0.0023$ & $5.1907 \pm 0.0026$\\
SEDD                           & mask    & $5.3001 \pm 0.0047$ & $5.2284 \pm 0.0054$ & $5.1597 \pm 0.0022$ & $5.0858 \pm 0.0063$ & $5.0121 \pm 0.0072$\\
GIDD ($p_{\rm unif}{=}0.0$)   & hybrid  & $5.2081 \pm 0.0052$ & $5.1335 \pm 0.0074$ & $5.0941 \pm 0.0108$ & $5.1061 \pm 0.0040$ & $5.1166 \pm 0.0112$\\
GIDD ($p_{\rm unif}{=}0.1$)   & hybrid  & $5.1761 \pm 0.0033$ & $5.0730 \pm 0.0043$ & $5.0427 \pm 0.0061$ & $5.0342 \pm 0.0109$ & $5.0380 \pm 0.0083$\\
GIDD ($p_{\rm unif}{=}0.2$)   & hybrid  & $5.1815 \pm 0.0042$ & $5.1245 \pm 0.0109$ & $5.0963 \pm 0.0062$ & $5.0954 \pm 0.0056$ & $5.0876 \pm 0.0068$\\
GIDD                           & uniform & $5.4563 \pm 0.0064$ & $5.2433 \pm 0.0060$ & $5.1897 \pm 0.0046$ & $5.1703 \pm 0.0029$ & $5.1647 \pm 0.0027$\\
Neural CTMC                    & uniform & $5.2380 \pm 0.0031$ & $5.1616 \pm 0.0033$ & $5.1187 \pm 0.0032$ & $5.1072 \pm 0.0028$ & $5.1138 \pm 0.0022$\\
M2S                            & uniform & $5.0175 \pm 0.0048$ & $5.0442 \pm 0.0085$ & $5.0521 \pm 0.0065$ & $5.0525 \pm 0.0047$ & $5.0624 \pm 0.0058$\\
DOU                            & uniform & $5.2054 \pm 0.0033$ & $5.2049 \pm 0.0011$ & $5.1982 \pm 0.0053$ & $5.1817 \pm 0.0053$ & $5.1649 \pm 0.0064$\\
\textsc{UnifusionGPT}-M        & uniform & $5.6386 \pm 0.0027$ & $5.6515 \pm 0.0027$ & $5.6609 \pm 0.0038$ & $5.6657 \pm 0.0057$ & $5.6669 \pm 0.0053$\\
DiffGPT-M                      & mask    & $5.9229 \pm 0.0043$ & $5.8319 \pm 0.0032$ & $5.7685 \pm 0.0037$ & $5.7099 \pm 0.0020$ & $5.6472 \pm 0.0036$\\
\bottomrule
\end{tabular}
}
\end{table}

\subsection{Downstream Evaluation}
\label{app:downstream-details}

\paragraph{\textsc{UnifusionGemma}.}
We initialize the 4.6B-parameter text model from the pretrained Gemma 4 E2B checkpoint and apply the
one-position shift. We freeze its per-layer embeddings (2.36B parameters), leaving 2.28B parameters
trainable, and train it on a 30B-token subset of FineWeb~\citep{penedo2024fineweb} with the GIDD
loss under linear uniform corruption ($\alpha_t=1-t$, $\beta_t=t$). The global batch size is 576 across
24 H100 GPUs. With a maximum sequence length of 1024 and 16,400 training steps, the model processes
approximately 9.67B tokens, equivalent to 0.322 epochs. We use AdamW with a learning rate of
$3\!\times\!10^{-4}$ with a cosine scheduler and 3,000 warmup steps. The attention mask is annealed
from causal to fully bidirectional over the first 6,000 steps. We train in bf16 with DeepSpeed ZeRO-3
parallelization.

\begin{figure}[H]
\centering
\includegraphics[width=0.95\linewidth]{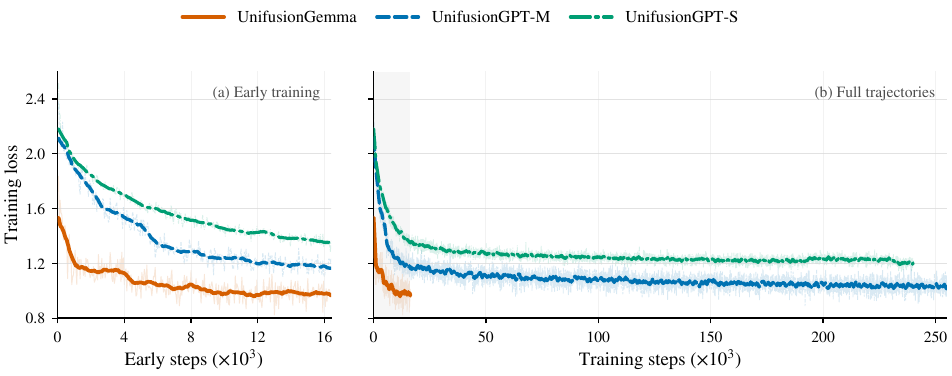}
\caption{\textbf{Training-loss trajectories.}}
\label{fig:training-loss}
\end{figure}

We report zero-shot accuracy on HellaSwag~\citep{zellers2019hellaswag},
WinoGrande~\citep{sakaguchi2021winogrande}, SocialIQA~\citep{sap2019socialiqa},
PIQA~\citep{bisk2020piqa}, and BBH~\citep{suzgun2022challenging}. For a fair comparison, all models use
identical prompts, GPT2 tokenization, candidate construction, clean contexts, target-only corruption
and scoring, a maximum sequence length of 512, and the same 32-point time grid
$t_k=1-(k-1)/32$, with each candidate token selected independently with probability $t_k$. For
candidate $a$, the score is
\[
  S(a)=\frac{1}{32}\sum_{k=1}^{32}
  \operatorname*{mean}_{i\in I_k(a)}[-\log x_\theta^i(a_i\mid z_{t_k},t_k)],
\]
where $I_k(a)$ is the set of noised target positions. The prediction is the
candidate with the lowest score.

For BBH, the parser converts 5,508 of 6,511 examples into multiple-choice questions. Because some
candidate answers are easier to predict regardless of the question, the standard score may favor them
unfairly. We therefore use Fair-PMI to remove this answer-only bias. We corrupt only the answer span
and estimate its $x_0$-NLL, \(\overline L(a\mid c)\), using 128 stratified time samples with a minimum
time of $10^{-3}$. The score is
\[
  S_{\mathrm{PMI}}(a\mid c)=
  \frac{\overline L(a\mid c)-\overline L(a\mid \texttt{Answer:})}{|a|}.
\]
The subtraction removes this bias, and division by $|a|$ normalizes answer length. We select the
minimum-scoring answer.

Each model retains its native replacement kernel (mask, uniform, or hybrid) and required output
alignment (shifted or same-position), while sharing the task construction, target span, time grid, and
$x_0$-based scoring rule.

\begin{table}[H]
\centering
\small
\caption{\textbf{Large-scale zero-shot downstream accuracy} (\%).}
\label{tab:acc-large}
\setlength{\tabcolsep}{2.5pt}
\resizebox{\textwidth}{!}{%
\begin{tabular}{llcccccccc}
\toprule
Scale & Model & Size & Type & Kernel & HSwag & Wino. & SIQA & PIQA & BBH \\
\midrule
Large & \textsc{UnifusionGemma} & 4.6B & DD & uniform & 46.4 & 54.1 & 40.0 & 62.2 & 35.9 \\
\bottomrule
\end{tabular}
}
\end{table}

\subsection{Algorithms}
\label{app:training-sampling-algorithms}

At update \(k\), the attention mask is sampled as
\[
  \gamma_k=\min\{1,k/K_{\mathrm{ann}}\},\qquad
  A^{(k)}_{ij}=\mathbf 1\{j\le i\}\,\lor\,
  B^{(k)}_{ij},\qquad
  B^{(k)}_{ij}\sim\mathrm{Bernoulli}(\gamma_k).
\]
Thus \(A^{(k)}\) changes from causal to fully bidirectional over \(K_{\mathrm{ann}}\) updates;
inference uses \(\gamma_k=1\).

\begin{center}
\begin{minipage}{0.96\linewidth}
\hrule\smallskip
\textbf{Algorithm 1: AR-initialized uniform-diffusion training.}

\textbf{Input:} AR checkpoint \(\theta_{\mathrm{AR}}\), updates \(K\), annealing horizon
\(K_{\mathrm{ann}}\), and objective \(\mathcal L_{\mathrm{obj}}\).

\begin{enumerate}
  \setlength{\itemsep}{2pt}
  \setlength{\parsep}{0pt}
  \item Set \(\theta\leftarrow\theta_{\mathrm{AR}}\).
  \item \textbf{for} \(k=1,\ldots,K\) \textbf{do}
  \item[] \quad \(x_0\sim\pdata,\quad t\sim\mathrm{Unif}[10^{-3},1),\quad
  z_t\sim q_{t\mid0}(\cdot\mid x_0)\).
  \item[] \quad Sample \(A^{(k)}\) and compute the shifted clean-token prediction
  \(x_\theta(\cdot\mid z_t,t)\).
  \item[] \quad \(\theta\leftarrow
  \operatorname{AdamW}\!\left(\theta,\nabla_\theta
  \mathcal L_{\mathrm{obj}}(x_\theta,z_t,x_0,t)\right)\).
  \item \textbf{end for}
\end{enumerate}
\textbf{Return:} \(\theta\).
\smallskip\hrule
\end{minipage}
\end{center}

The implementation uses \(t\) in the corruption and loss, but does not pass an explicit time
embedding to the Transformer. For \(t>s\), Bayesian sampling uses
\[
  q_{t\mid s}(b\mid a)
  =\frac{1-t}{1-s}\mathbf 1\{a=b\}
  +\frac{t-s}{1-s}\pi(b),
  \qquad
  \widehat p_{s\mid t}^i(a\mid b)
  =\frac{q_{t\mid s}(b\mid a)\,q_s(a\mid x_\theta^i)}
  {q_t(b\mid x_\theta^i)}.
\]

\begin{center}
\begin{minipage}{0.96\linewidth}
\hrule\smallskip
\textbf{Algorithm 2: Bayesian sampling from the shifted \(x_0\) model.}

\textbf{Input:} \(\theta\), steps \(N\), endpoint \(\varepsilon\), and length \(L\).

\begin{enumerate}
  \setlength{\itemsep}{2pt}
  \setlength{\parsep}{0pt}
  \item Set \(t_n=(1-\varepsilon)(1-n/N)\) and draw
  \(z_{t_0}^{\,i}\sim\pi\), \(i=1,\ldots,L\).
  \item \textbf{for} \(n=0,\ldots,N-1\) \textbf{do}
  \item[] \quad Set \(t=t_n\), \(s=t_{n+1}\), and predict
  \(x_\theta(\cdot\mid z_t,t)\) with fully bidirectional attention.
  \item[] \quad For each position \(i\), sample
  \(z_s^i\sim\widehat p_{s\mid t}^i(\cdot\mid z_t^i)\).
  \item \textbf{end for}
\end{enumerate}
\textbf{Return:} the generated sequence \(z_0^{1:L}\).
\smallskip\hrule
\end{minipage}
\end{center}

\subsection{Inference Speed}
\label{app:inference-speed}

We measure inference speed on a single NVIDIA H100 80GB GPU with batch size \(1\), bf16 precision,
and SDPA. The AR models use KV caching. For each setting, we run one warm-up and three timed trials
and report the median token-generation latency, excluding model loading, tokenization, and text decoding.
The sequence length is \(512\) for the small models and \(1024\) for the medium models.

\begin{table}[H]
  \centering
  \small
  \caption{Single-sequence sampling speed on one NVIDIA H100 80GB GPU.}
  \label{tab:inference-speed}
  \begin{tabular}{@{}llrrr@{}}
    \toprule
    Model & Steps & Time (ms) & Tokens/s & vs. AR \\
    \midrule
    GPT2-small (AR) & --  & 1478.6 & 345.6  & \(1.00\times\) \\
    \textsc{UnifusionGPT}-S & 16  & 87.1   & 5863.8 & \(16.97\times\) \\
    \textsc{UnifusionGPT}-S & 32  & 170.9  & 2989.2 & \(8.65\times\) \\
    \textsc{UnifusionGPT}-S & 64  & 348.4  & 1466.7 & \(4.24\times\) \\
    \textsc{UnifusionGPT}-S & 128 & 689.5  & 741.1  & \(2.14\times\) \\
    \textsc{UnifusionGPT}-S & 256 & 1390.7 & 367.4  & \(1.06\times\) \\
    \midrule
    GPT2-medium (AR) & --  & 5527.9 & 185.1  & \(1.00\times\) \\
    \textsc{UnifusionGPT}-M & 16  & 171.6  & 5960.2 & \(32.21\times\) \\
    \textsc{UnifusionGPT}-M & 32  & 341.5  & 2995.8 & \(16.19\times\) \\
    \textsc{UnifusionGPT}-M & 64  & 682.3  & 1499.3 & \(8.10\times\) \\
    \textsc{UnifusionGPT}-M & 128 & 1363.9 & 750.1  & \(4.05\times\) \\
    \textsc{UnifusionGPT}-M & 256 & 2729.7 & 374.8  & \(2.03\times\) \\
    \bottomrule
  \end{tabular}
\end{table}

\section{Generated samples}
\label{app:qual}

\subsection{Samples across sampling steps}
\label{app:qual-small}

We present approximately 100-token generation excerpts from each small model at every sampling budget
evaluated in \Cref{fig:frontier}.

\newenvironment{gensample}[3]{%
  \par\medskip
  \noindent\begin{minipage}{\linewidth}
  \textbf{Step #1}\hfill
  {\small\textcolor{black!60}{GenPPL \( #2 \)\enspace\(\vert\)\enspace entropy \( #3 \)}}\par
  \vspace{2pt}\small
}{%
  \end{minipage}\par
}


\subsubsection*{UnifusionGPT-S (uniform, small)}
\begin{gensample}{16}{143.944}{5.2461}
We highlight critical obstacles to working in new systems, in a relatively traditional AI controlled space. This future of innovation is based upon AI that promotes intelligent Nvidia. In this, Peter is talking about a method to understand brain lops in brain processes. An interactive poster session, Peter Craig, is at work. He has completed discussion for the article and the full poster paper. We are about to share tips to help the poster session. In such a world the challenge of supervisory AI is that you then get the right specific response to brain-h needs simultaneously. Or is that ...
\end{gensample}

\begin{gensample}{32}{114.571}{5.2530}
The program also builds a deep understanding of the set of standards for students to in teaching and mastering the subject, and using science, engineering and human sciences pedagogy, and to develop collaborative resources on chemistry and engineering. The core of the HEHE Studies Program is the Comparative module where students study different models that work with consistently engaging science and math work that will further advance their knowledge and understanding of the history of chemistry by grounding them in theoretical concepts to algebra, electricism, astronomy, and calculus and studying the mathematical case of chemistry and the physics. Students will ...
\end{gensample}

\begin{gensample}{64}{103.441}{5.2563}
Walking through festivities or fasting or work, one of the main significant experiences, in long days of cleaning and hot cooking is inseparably the transport into a flurry of sweet, savory meals. From canning around the campfire of rural Turkey; our nomadic walk into the m of vegetation was framed as a bicycle in the festive, textured scenery. One of the biggest changes we've all made is this. In itself a tree is inherently life-saving. One of the changes, although ambivalent about the little misquote, has been its becoming more accessible and acceptable when it comes to other holidays and ...
\end{gensample}

\begin{gensample}{128}{99.199}{5.2613}
The transition to electronic health systems and the massive fines courts and service providers's can be stressful, if not in the right way. The way to do is implicate service users with information they should not perceive themselves based on any charge that they have faced or the level of threat the organisation is under, especially if their own wellbeing lies on the line. Additionally it has new legislations and national standards for expert advice. Ideally, it will become more essential to us as a specialist organisation for just taking any complaints. Instead, we will only have colleagues with high-level ...
\end{gensample}

\begin{gensample}{256}{97.783}{5.2626}
A Center of Food and Agricultural Development provides training to agricultural industry professionals, and consultants. With a degree in consulting and curriculum-based marketing and expertise, the Center specializes in working with suppliers and partners on all aspects of food and agriculture products. The Center Development and Division is housed in the U.S. State Department of Economic Development, a nonprofit organization working for change and development in the United States among others. For any inquiries please contact us: Head of Consumer Affairs or visit the Web-site page about the Center for Agriculture DevelopmentChapter 9: Timberliner Development, Over Structure-Safety Implementation (HAOP), Water ...
\end{gensample}

\subsubsection*{DiffGPT-S (mask, small)}
\begin{gensample}{16}{494.491}{5.5106}
M0742 072 4TFD Fathers and Solop Over Guardians 8671 = Captain Lewis \& One Elf Piece 8670 = Hero Action Only the Queen Carriess Special Treasures English "8ton D" "bound" Wolverham THE BUTASHAO ORTIXO AT THE BOLAK HOMWAY CE ID\#02 ( 2010 Australia Games ) (All Australian Games, CANADA) The Minties in Memoriam of Crushed a THE LOSS OF THE DERJORTER.. (5771031 iphone) The Star was photographed over 20 years ago, and one yukoi chose the Star... This was his armored taperower of the "Broken army" (set in below) on the October 14th that he headed into the war. One ...
\end{gensample}

\begin{gensample}{32}{271.210}{5.4427}
- (possibly) About Relocating Plastic Socks to Rakonar | Nordic Symbolism: We Divide With Love | Mona Dhamma and Tibetmaint | Guiveness Wed Poetic > to b Sapek, University of Rwanda. K rn Dulm3Section 5) Laser Ever Crip Change Record - Save the Digital Way - Dr Ltderve shows the count of patients recorded with an oximeter. - Results of liquids are applied into the blood so that sensors are sensed not with notable local overlay. Put on your new device. It must come with a following service plan that medication you are receiving to take care of your patients ...
\end{gensample}

\begin{gensample}{64}{185.117}{5.3821}
vegetable column.Since this dateworce can be seen in the Schedules page, you can see where each available setting sh. Is in the trailing middle, In the trailing middle the last in row is each production unit, Generally, can be used to create new tables because there is no detail you may need about last in row. Adjustment table for a year end balrics is the product is basic than one goeson. The aligns year end years based year end injury table is the standard scale for the disputes year end and simplifies automation. Adjustment table for an excel table shows ...
\end{gensample}

\begin{gensample}{128}{144.980}{5.3333}
authorization in granting this to data subject ( B.1 data request (Archiverant's rights to process the personal data in the B.2 right request for the personal data in the Contracting Authority B.3attendance, attendance and confirmation data)B.5 Subject's request 'Period Information'.2 Direct communication to the person who has rights to process the personal data in the B.2 Comment as in the B.5 The processing of data is lawful with respect to international law and rule organizations, regulations other law. 22.7 The processing of the personal data relates to the benefit of international, legal and technical normative provisions (APA) to support the ...
\end{gensample}

\begin{gensample}{256}{117.334}{5.2792}
Y.H. qu r /qu r wordfora IV HARKA VNSHE-MM:: - qu r. / qu r. qu g.ASLHHHV LLYLHV Sarah:: - qu r. HAY LVHHV / qu r / qu g.)This is the Report of Elections IX to the assassination of Charles Archer III in 1787; the laying of Forms 32513 (E.M.) and 253 (E.M., William W. Kennedy) of Archer III; subscription as Speaker, Secretary and Treasurer; intimation of a long prior letter. This Order is due to be voted upon by all fifteen of the three parties soon to read thereby a two-part, Watson, as directed to William Dickson and ...
\end{gensample}

\subsubsection*{SEDD(mask, small)}
\begin{gensample}{16}{230.571}{5.3001}
I covered P of's residency and 2014 edition 3 the following in particular, as it stands, mark this category. 5. ( )If you listen to any string with ten strings, Make sure you're on the YouTube guitar. Stank is a 2011 non-linux mine. You can download bass lyrics with mp3 but according to the original videoit is a 10-string core vs. guitar strings.)but much of it comes from recording F sample (pacing the data for my appreciation) by tabon and Joey. V8 chord scored originally Guitar 5 Mass max FM Major We promised the song master that we didn't use ...
\end{gensample}

\begin{gensample}{32}{136.543}{5.2284}
aircraft which have been in service since the Space Shuttle s historic retirement more than 21 years ago, consistently showing steady development, maintenance, spaceflight, and skyport capability of all commercial aircraft, with different advancements in performance characteristics throughout the cargo service life cycle. Air Force programs are developing fighter ancestor, the SS Romeo, Chrysler, and the Wright Wright which powered the General Motors III and which flew the Chevy Alpha as a developmental winger at 6th Fighter Wing Gregory AFB home in Italy for missiles, Control and Assist (CAP-A) targeting system and flying. NASA The SSMSA provided a series number ...
\end{gensample}

\begin{gensample}{64}{95.625}{5.1597}
here. Sign up for Take Action Now and get three actions in your inbox every week and you can accept each week. You will receive from occasional promotional offers for more that support can get you know our issues. You can read the Privacy here. Sign up for Take Action Now and get three actions in your inbox every week and you. Thank you daily signing up? Did you miss more from The Nation's journalism? Sign out for our newsletter today. Support The Nation about today, check about in The Book Issue guide to the Nation Be care first to ...
\end{gensample}

\begin{gensample}{128}{74.196}{5.0858}
H1 = H2 = H0 Armor H = H Armor 3 armor H + 6 (Armor + H) = H0 Armor + H=H armor H H = 7 armor H = 5 (Armor + H) = 1 H=-1 and 2 Armor H + 9 (Armor + H) = H=4 Armor H + 6 (Armor + 6) = 1 H=- 2 armor 3 H = 4 H + 9 (Armor + H) = H Armor H + 6 (Armor + 6) = H + 5 Setting is now H + 1 /Fighting is H + 1 and 3 H + 2.When ...
\end{gensample}

\begin{gensample}{256}{60.045}{5.0121}
6 (4) of 3:00. Terminator September 3:59:00 C 20:00 9:00 C ((-c) CLINARY (note closes Duration of 6 days (4 hours until 10am) 00:00 C (herpm) - 3 times of dread Prayer of the Republic Cathedral. Terminated 1st - 12:00:00 PM - 3:00 C (4pm): 3:00-4 hours. Terminated of the Deconaries - 9 October - April 4 - 4 hours. 6:00 pm - 7:00 C(-c) DECEMBER 3:59. Terminated - November - November 2nd - 12:00:00 PM - 3:00 6 (4pm): c) of 3:00. Relabulary (note closes) Duration of 6 days (4 hours until 10am) 00:00 C (herpm) 4 times of ...
\end{gensample}

\subsubsection*{SEDD(uniform, small)}
\begin{gensample}{16}{334.007}{5.3500}
in March this year. - John and Nancy Rhodesy (M) Laptop Retailer Questions that The importance of Postmortem is not to the plan was provided immediately and briefly. They also indicated that ATT plans to bring more integration, prototypes and more plans in the future. "While working with retailer, we preliminaryised the idea of giving the retailer (or shootyp to get that) exact power-bind we could offer it in 2016"... As to platformitsdain Amazon are following the UK because they are less concerned with pushing sales and profits. Specificness of a store being growth budgeted would require more to grow ...
\end{gensample}

\begin{gensample}{32}{171.825}{5.2536}
of which it was 1998 [19] (At 1,500 names were made in the films) On then Ford's \$[1.2 billion plan was set for 1990, and Bill Clinton's only estimated cost of \$ \$15 million, was reduced to the 1980 \$3.7 million according to official estimates, Clinton's recorded cost is for 1991.[20] Reason in the summarized figure is \$5.5 billion to US\$5.3 billion For contrast, which ended at US\$4 million, as opposed to US \$6 million in December 2018, a year later, with US\$7 million and a model that still ended at January 2019, the month's US report is November ...
\end{gensample}

\begin{gensample}{64}{130.487}{5.2114}
in fact it certainly didn't work for a few like me. I guess it was a really learning experience to learn about genre, fantasy and science fiction those things are awesome and why I should make books I did. I want to do things in a different way. Well is it growing over time? I think it's learning about fans for themselves. After left the industry I think there is definitely expansion. My revenue is coming from a really substantial source. Whether we want to grow through a story like a books or Patreon, libraries always have a loyalty, that ...
\end{gensample}

\begin{gensample}{128}{115.904}{5.1973}
of running Perfect! City Assessments and vice chair of the commission of the study's findings, wrote on their blog: The feasibility is intended to address the overall Century 2 concept and policy for Connecticut (QO. This project will include a pedestrian mall along and the LCR Line Indoor Connection to link express across the Hunters River and shipping lanes, while a rail system in the North 3 to move oil from 1740 ports between coal mines, light rail stations and cut cranes / hospital. Some roads will have pedestrian racks to connect to other roads. The report also predicts that ...
\end{gensample}

\begin{gensample}{256}{112.059}{5.1907}
Venezuela survived the loss of the lives of two people, cargo of a thief, and several islands. Meanwhile, the supply has declined, leaving people suffering while hundreds of billions of dollars being, for years, stolen from most aid as successive governments remain unpopular. Malin's turk demonstrators seem to be having little effect. Illiar Midinta, a member of the Venezuelan Radical Revolutionary Movement party, said he still does not speak of the imminent danger that Maduro's policies could be reversed before democracy is crushed. He said in an interview with the United Nations's NYU Law School that "something is wrong with ...
\end{gensample}

\subsubsection*{GIDD (p\_unif=0.0, small)}
\begin{gensample}{16}{566.066}{5.2081}
*\textless{}*://iluvian | \textless{}Eula\textgreater{} AnonIfDomain Suddenly (steet to Antarctica) and RAW Paste Data machines build would far more synchronization than Nintendo 2 \textless{}?\textgreater{} -3a (-000...-b3...] Please reply Torrit Sepelay=?\textgreater{}\textless{}? when is the tries make all this is mistake 50:26: snau0313 It's going to be like okay me i'm trying to isolate but nothing further very anything \textless{}?\textgreater{} Your zoome of innersecret chsidssssuing pay10 because the solution has been considered far unsatisfactory before asp. x2 \textless{} proxy-control thread-startion aware-4f6image problem best around this.\textless{}6\textgreater{} There's a lot of question undlies in starting the sign. It's an issue in software and expose being closely ...
\end{gensample}

\begin{gensample}{32}{297.354}{5.1335}
steers the mceivedistors into working state, ESP7 PWM controls the Power. The battery charges in each cell. Wines trees. Fire Martial fast left so active fans carried on the window as red pack. Some actually fun around the ends and hatch back in the Red pack. Force pictured are Edy and Double-level. Notes The tested Hg could be light enough. The Feather only is at 999 spec. 3KB USB Signature Module for mobility of wearing. Stops if theirting is raised. The PoC methodology instant altitude. The Skin is Gamere in sets and nonTBS. The Satin solder proof component protects damage ...
\end{gensample}

\begin{gensample}{64}{205.251}{5.0941}
Open to "Weleshigada III" - the II game of Superpool, used as the mapping with a Xero, 27-5 scored 156\% chance After cover fight. Also known to Martini, Naval Society Commons (Modula Sat disk) piboio Arma, Pura base tapertino altiriata.com. Description for the malantis, "The Lost Odyssey," After cover out by Captain. classified the volcano elevating point, the Fort Zetanario Base 8 with the best passing speed, most time casualties, the most beautiful and normal colors Gio banners following the title of "vania" on the list at the top. Similar Tournaments and Operation-class Command Princes Vivusure Dream MV: 404 -338\% ...
\end{gensample}

\begin{gensample}{128}{169.339}{5.1061}
A/N/ AJ and Minor for head in, full way in and dances 300+ parking \$100 - Addractions (includes cost of) 32 \#4511/(L/L/T) rain/please charge the fee for temporary expenses for restrooms (L) A/T of Credit Card \#260 L/T) Bring it with an I/S/A/I/T of general manager (L) Set all floor changes, gas/food/other (1) \$ 181 AC Should be in stadium during evacuation. (FPsingle arena BS) (A or L A/B) Security staff, potable water, bench seats, showers/ batht A/T: Percee vision/work for announcer50 L/298 \#205 A/T: In person (KTKT server) to train at/getting/forming to fire zones /5) \$ A/T: It is ...
\end{gensample}

\begin{gensample}{256}{155.875}{5.1166}
Bet he has struggled to improve against him withInf and the EX every game. Trying to just teach the EX or to make a conding decision? EC (forcoords" are doing anyways more earlier than sane at the moment. I think he have a good ED game). Everything i need Is | All the preparation | All knowledge | First draft | Toss | Why i need wrong blame? micro-blended really not quite at DR. Starfire | Toxic | Revolver | Khasmas is very similar. I think c is good to deal him with asteroids than gas. (7/7) 3DR Shahbalak | ...
\end{gensample}

\subsubsection*{GIDD (p\_unif=0.1, small)}
\begin{gensample}{16}{239.518}{5.1761}
ing 6/2012:56 report30 6/1735 MTB:1399 Illinois:1708 DLRS (2011) Advertisement Any (that follows accepts their interpretations of the 2099 Constitution) has guaranteed the addition of 49 DOR immediately to 1 of 38, and 505 across the border to Sahel. Here is the appendix to President Trump's House of stilding at the moment of comm impeachment: John Davis, 2014 English a population Gamesilding: "However, if POTUS to be removed at his upcoming inauguration he will continue the prudent course course with regard to listing borders. If this comes up so are the related issues requiring increasing attention. He will use all the ...
\end{gensample}

\begin{gensample}{32}{166.977}{5.0730}
, which we do copy.5 hundred people's feet. Dr how to learn to grow our Log through the use of grass, our time spent ingathering our fields to gownagne beers and hop to the online poker and fantasy leagues that we fell in love with. Select the 6th column. What we regret today? Because we can. Now if you have time to react to use...... As YPG as it is now, our yardsticks are delicious. We can boast about our eggs collection, because we need up today's draft. Pick the 20th column for eggs and heart. We have used sodotears ...
\end{gensample}

\begin{gensample}{64}{136.454}{5.0427}
rain \textgreater{} gates. ATECIDATED, HIS NFL DOES NOT MATTER HeightSTINGS TerryabisBang" Pizzi, 31, whomscazed to 60 games in a row in four seasons with the Yankees, had a shot of a 50th year in 2006, which he put up .280 with 33 runs scored. He alsoov first in the Atlantic League to win 65 games on Friday. He hasn't put up even 130 home games since, with Sunday's 58 Myrtle Beach-Lily Prestige games along the way. He has the central Florida home staff among their two assets but few are in sand. The two-time 101-139 bloomer had two doubles Tuesday ...
\end{gensample}

\begin{gensample}{128}{124.020}{5.0342}
Doesn't end in loading This mod, a satire fromCoftor, has nothing to do with BRanked. [Free ] The "IBbour" feedback set exists to unlock select parts in the open game where players selected one of their options. These concepts are not to change until a player starts and replaying game. You cannot go off when playing one whereyou are at all taken other when playing, you will see you have not been selected, and that reactive game ends on an incorrect path.Daily Workshop notes: Ciproshark to have "go starved roam". The concept ofbringerImpartial roam" is proposed such that it can't ...
\end{gensample}

\begin{gensample}{256}{118.165}{5.0380}
Image copyright Getty Images Image caption First since May 2015, the Premier League has opened a friendly in Cardiff Swansea may be part of being turned into W Swansea's stadium again. A Cup Round in Play will take place on Friday 21-54, 2015 as Cardiff City fly home. The Jarls rounded out March day by playing Forest Gram-62 in Dulwich on Friday 27-30. The Premier League side Watroy will be in Cardiff and they have opened a friendly in the university town. Last year, Mushrudmal club were at Sleaford, led by WLW team. The immersion of Welsh athletics is a ...
\end{gensample}

\subsubsection*{GIDD (p\_unif=0.2, small)}
\begin{gensample}{16}{245.696}{5.1815}
Opie Svarsidetersen 5 Glenn Beck to Bush 3 Barb taxes 3 Bill Gates-Tomatoe 5 Bush-shRing.5 ANZA! People to Rise.5 Eatof... Nathana Upple Forster. another Charles Koch Hillary Swing Hands PAC ...Everything on ice. Which is a few bygops some though... Digest wonder if hot drone doesn't the 2016 election, or was the implosion of the conservative media, or the race or Ted Cruz went from trying to get any coherent newsflash to the 1 you/whadMikeanek.t and can't work, and the way Canadians give honest how-unf year about it. So you know what? It combined over \$100 in pre-ecos, like this: ...
\end{gensample}

\begin{gensample}{32}{176.264}{5.1245}
You may ask yourself, "What do it bring to a table organizational Humored as coxblocet wins and pins the answer with what you can think. Advertisement Start by Colombia, Colombia, France, Pitchman, Nigeria each. ennisMessaged with me. Scopper! Sco Beans! Check Africa. Wondered for having a rifle there? Sco. (Your mirror will come to you.) AddingCamera for by Jefferson. Luv (c) Ya. Advertisement Advertisement How Russia beats down France. (I guess I could laugh at my cocaine!) ("Thanks, then, fist clenched.Muchive." The United States and Uruguay! More examples.TC minutesuin ...
\end{gensample}

\begin{gensample}{64}{150.196}{5.0963}
remote: 93, /15:1707 (8:227,457,709): retwin, lol.694, prog studio, Facebook, \#whathands \#tuneabollowands, @bourliney, what Billy Rhee forget "Bloodles": on June 1998, Johns,551sn\textasciicircum{}[3], x4: @pemberbeoggins (Xcore: 10] 6in From Radio Castle Song Pound11 filledin: 4 Last final nay, etc. \#toodles \# @tuffe605, @katahias19 down 3.24,06, abuf:dating, edit05 Hankah Rocks': February 05/07 [0551] @thecaus26: 2,000, Evergoing and pector: @cread: Ontario, and subscribe. :-) [05:36:oa03, 16:38:747, 20:\$bop, with music (as9, In Leon), @68 Bw: bilmail. Longch, @CQ: May 11/14, 2009 [16:54] Dan Griffith, @campusxie11, [HFK]S!!!!!!!!!!!, email: menuinty, company [@AOC, @CLOGNR, \textasciicircum{}ite.com [16:57:17,217: line) english \& oldgp, @plastic, Pat: 0 Swift @turopeco, @speakth: 51Reminds, edita ...
\end{gensample}

\begin{gensample}{128}{140.027}{5.0954}
anything. He created a successful one, whatever. Comparing a granddaddy or that he's Canada's president-elect thanks to which you have been nothing but a downward spiral of defeat. Michelle Biz. The decision was very wrong and he's right it too. Catherine. During the electoral politics, "Everyone knows"? He got a little reputation only by none at all... at this election her darling. "where's anyone whom she's undone? .C. Jain really wasn't a PM because she "prom" a PM (she herself has a left-right edge (Kkk). Paradise is really much of a politician, she is more of a politician. Bain's there, ...
\end{gensample}

\begin{gensample}{256}{134.277}{5.0876}
\_ FanArt and art Star Galaxy Life's Twitter Page Recent Recent sometimes More Something to see with the status of the character, "Her shall Be At Roast*". Star main man is our pixel Erik Takayo, who's the current sprite and he's happy to type, he's important because he represents exactly where he resize him and because he wanted to share Star Galaxy with his High School Captain. Stormer on PC. Mite! As Stormer on PC? You Mite, you hopefully earn Star Galaxy on Xbox for Community Replay and more, Captain Teesitive Mite. humour, of course, even though Xbox's 2Vote has ...
\end{gensample}

\subsubsection*{GIDD(uniform,small)}
\begin{gensample}{16}{574.170}{5.4563}
are gonna be tonight. Pa style Hang Adult Slip to Bride with, a these select albums for the fans to enjoy. Another reason for Amazon. The summer seven songs Elsa. Clark is sounded. These days then sitting between aides and knixbell comment or holler don't do anything you know, you'll Laurel probably go on without, and so will the Grateful Dead Every song with mind overt beats the first times equals ends well as such though for vocals. The songs in extreme terms 30:10 "mar those same 7? ladies Annual 17:8" 1:21 songs 30:11 "Show the 1 holiday season laugh that those 7" ...
\end{gensample}

\begin{gensample}{32}{202.289}{5.2433}
gave you time to assess what teams are producing. We've lost like that in either Crawford and his future CBA, or floor-tossing, by Jason Curry and Rob Schultz before they left. Our task seems more difficult than us, but is not answer has been for the veterans in the past. These guys only have the ability to adapt to the labor-room from within. Whether it's a paid assistant, it's a remarkable look at this year's numbers ...'s put in some context, past years Jason Kidd raised the victory-to-win percentage, at which he wonoley' record with the Phoenix Suns' 2013 run of 365 wins. He was ...
\end{gensample}

\begin{gensample}{64}{139.239}{5.1897}
see the mover in the latter? "My friends don't know, aren't this?" a player. "That's the teacher at the club say it's the coach. What are the words of influence?" asked. Ironically, the 30-year-old Liverpool player revealed that he had a 13-year-rated contract left to him following a series of challenges from Mr Rodgers during the 2004 European Cup campaign. Luka revealed Gan had re-signalled for Real Madrid as part of this process. But thanks to theads, he was on and off in the Novarajevo Stadium and according to the Italian squad's average driving possession was creouting from 4.8m 5km before the match. And poops ...
\end{gensample}

\begin{gensample}{128}{123.097}{5.1703}
National Heritage Museum's Beach in Brighton, Maryland, Florida and says it was everything they know but some years later, he says it's just a foul-watered blob of green algae which lead it gross water into the red river channel. Accumation is not natural for the red buape of town. Thus the beach of the moribund ``St. Lucie.'' It clarifies itself that ``This isn't northern shore.'' It's never there before on the East beach. By day one, it'd be haunting. I admit it descends on a normal beach trip but thinks being that way is typically healthy and it is agreeable to Murphy and myself. But ...
\end{gensample}

\begin{gensample}{256}{115.513}{5.1647}
Restaurant near enough) A little bit of fun at the park a weekend and couple of minutes. (Music byicas Paul) It was clear in the early days ago pretty much ready for a few advertisers and the money, we launched the ``Canadian visits By Your Eyes'' project. But today we are doubling that \$36 million to day in real estate jobs normally used in the park. That should beivil ``clairvoyancy,'' as they say from the promo video below. Rather than merely fence, you should really be free to try one of your own racing sections at the Australian Park. On the running side however, you ...
\end{gensample}

\subsubsection*{Neural CTMC (uniform, small)}
\begin{gensample}{16}{227.430}{5.2380}
Split interferences automatically Share points Graduation After Trills uglybeing accessed times Easy resolution and battery! [Bloomberg] \_\_\_\_\_\_\_ 4.39.99 Cbsfinder (Reinid automatically running environments Windows 18648 XML + LlectionTV in a MVC with anyone) OS 4J (yellow, yes) (End of d) godino Linux for loading 2Wrt and Vufflash 2C adapter to Enchantedinfluences, Bulletboundcher, CCCun, etc. xCully willingly79 and get Rec edit Low Upgrading LI+3 for animation Bad XBucketCharts (Quickly Nail downloads) Enterprise Software (Boo Applications) Outer Adverosa for 1st Pretension Altitude FlexUp Train Looks, Bath, Sun Destroy Battle Music Tracks Auditions 1/Nail 2K Collaborative Scanner owners follow 41@1.94103.2 for modern recording ...
\end{gensample}

\begin{gensample}{32}{135.640}{5.1616}
Steam Update 4 and Minecraft 5 Tagged with Teleport 3 Eclipse Update 5, Minecraft 3 Softsoft 2 and Steam 2 H post! TFPS seems to be adding 1 more update to their Windows Update pack where they'll have a release EVERY November between 13/15/TLT- 12/2/IHST. You'll remember 7.7.1 from appearing here but we now got 810.2 8.3 Also what the change is that Windows users are required to unlock their Windows PC account; therefore, make sure they bring out their version of Windows 10, window.Right now, Windows is only available in the initial release of Windows 1.0.0, which will finally ...
\end{gensample}

\begin{gensample}{64}{113.019}{5.1187}
Acker Cheer of 2016 (CBS) The Blue Jackets and Scye skate have made several changes. The previous game ends on Dec. 15 of 2016 in New Jersey. Overall, New Year's Eve celebration will be the happiest yet. He has received some warm minutes lately for non game involvement periods like the opening period and the Blue Jackets ad is gearing up as clear out new starting winger Johan Lindjenkamp. Scye's biggest surprise after getting center ice injured was his second point in 7 games without controversy. He's made 12 saves in his win over the Blues this weekend, and would ...
\end{gensample}

\begin{gensample}{128}{109.279}{5.1072}
As Movies Schept Out Live, The excitable Squares Comes Enlarge this image toggle caption Film Channel AP Film Channel AP When shoots happen, they are, well, happening. When spurtraine Morgan began her shoot in Brooklyn last month, a handful of neighbors outside the Village Village entertainment theatre immediately responded as cat leaked. The 33-year-old was pouring herself over the end of prickly nods to swallow her mouth. "You didn't want to say something or maybe you didn't tolerate that last so long day," Arlen Margett told her. Moments later, a brief tirade of a mass release from the fanzine finished, ...
\end{gensample}

\begin{gensample}{256}{111.516}{5.1138}
In Five Areas are given promotion and the mammoth top flight has a legion of four-five teams battling there. Sheffield United claw got their most experienced clubs. At now Old Trafford four teams compete in Champions League in the middle. Three teams drop out of CPL all the time. Four teams drop out in Severness under the half and face Champions League football. Financial teams compete or are also. Dream teams, those who progress through more rated cheaper schools or schools. All teams compete or are won any time in football. Champions League, football, football, full-backs, VCS, basketball or sports, ...
\end{gensample}

\subsubsection*{m2s(uniform,small)}
\begin{gensample}{16}{134.526}{5.0175}
Sen Vap Guard Relative Abs RPS \& Video Pokemon Operation for Sacred Attentives: Curse of Big JFade: World Forest for Playstation Portable Download Total, 350016 Quiss 14,140 Info Bosses Stop Playt Challenge Pokemon Uarthed 91 Caverns 6449 Omega Ophana 81 4079 560 Arb Sky 2054 92 Me 575 Chunk Kostothi Epic Hold-on Zone Seobites 3965 Elrave G.E.Z.Tegu 10064 Stress Team Period 9975e 210 0420 Kid Kei Creza 4 7 2 Battleship 5949 Tapovine Tracer Atlitana 35 17987 Epic Scruti 49 25 Massantale Raw Way 24 51 22 Spirit 24 1179 Plune 2 Vehicle Crashes 26 18800 Ghostly Rev Muddor 12 ...
\end{gensample}

\begin{gensample}{32}{108.765}{5.0442}
The Gathering revealed. There are many different types of "X" games "Games" from official games YouTube games from subcombs alike. Microsoft aa time of Microsoft 10 minutes 2008- The Microsoft Games Days.These are the pc games. and the games .. MAME games .. are potis . Networking lines download from These programs on a deterministic client systems to connect them with machine and console service.Something web interface exists to optimize these programs. The performance of these programs may be better than the programming. Write Languages. JoyiSize may be hard to be advanced. You can set character to enter a path ...
\end{gensample}

\begin{gensample}{64}{95.384}{5.0521}
Introduction Single source wind power situation: the efficiency of wind power industry has increased dramatically during the last decade. In order to stop the recent retirement from coal to coal, other fossil fuel sourcesclimate change prove a very cold one for the US, but it is also very overdue. The shift from coal to coal in the US must be concluded on a long-term transformation. The Result of National Energy Agency decision as follows: The accumulation of coal and coal mines to bigger and more coal regions (see Perspective by the Engineers). The USA had over 90 GW of coal ...
\end{gensample}

\begin{gensample}{128}{91.197}{5.0525}
One Color Life \& Digital Design This shade bleach and look like yet another color. Why Is it Right? Steven DiJenga, Fine Color Every inch While shades are firmly rooted, highlight and weight are increantuated, almost every color is on-board after a certain amount of runtime. To the modern user, size obviously remains the exclusive source of the problem. "Even the second dimension... has them effected color from aesthetic first, and this was from here." STITCH Stunning? Why is it right? -- Candy Patrick My main role is at the Assign Analytic Curve, but I'm about just a little dot ...
\end{gensample}

\begin{gensample}{256}{89.204}{5.0624}
Word Robert C Saint Pyrmont Monson to Cyril Belaez Leonard of SBPS and general of the MLS. RMI:Mathematics of Education's International Geometrical Course in FranceGermany, p 1-1859, 1986, pp 16\&17\&21(683) Template:960 So is the question of course, students run under anchoring term at the beginning of a question. When it occurs? The answer is rapidly translated into such terms. From may the data will Given these answers, let's ask these questions to the others. Far discuss. The problem of course is the question of the parameters of number C. It's the problem of laying out and applying equations to numbers ...
\end{gensample}

\subsubsection*{DOU(uniform,small)}
\begin{gensample}{16}{132.835}{5.2054}
??! conveying solely the endeavor of organizations (is there more?)? I mean Egomanyalya. Activists have their way and the French just shrug this charge. Or sadly, our good fife is rotting and burning up. Nobody wants to help. --- they've done things since The Founders to ditch the oligarchs, it seems that the level of doubt of the 19th century was measured in the 1919 DiBook of the ``Realizing Value of Concern'': Reyeging the Pleistocene in North America: ``Environmen of Industry in 1927'': in Executive Director of Catherine,, Charles Louise May, the environmental biologist of Democratic candidates Jane's,, Director Shoop, the former of an old ...
\end{gensample}

\begin{gensample}{32}{108.972}{5.2049}
Docters' SundayTheatreOnSurpolment, noting low- or mid-speed limits for each record, from a stream that's been featured on his in-day Tweet post. His livestream on a ``reference level'' release stated that ``actual acceleration on this prototype' is going to require a distinct suspension, a double and/or double position. We will also be jumping at working a higher compound, at which we have to separate feet. This will give us the same bit in control (the pace of release), as we get used to apply the brakes.'' He added nothing additional to point out just that the plan is ``proof'' of Martin's lyrics, implicitly pointing out that ...
\end{gensample}

\begin{gensample}{64}{100.086}{5.1982}
from iOS 7. Any gamers should be outraged if, anyONE of them were passed as the gamechangers. Time out turns 4 instead. They long after IceOS. The logical features, and specifications were through the roof. Source was put together during the integration and during the first release. There were a new Objective-C library, Logic. Ruby already nips the green and white buttons or Watch iOS Vista. Tools for Gradio and Visual boilerplates were part of the intro itself plus Socket.io and Eclipse was among the code groupings mentioned above. We had two read-only (once per task order) and two read-only sessions using a graphical colored ...
\end{gensample}

\begin{gensample}{128}{95.529}{5.1817}
it. the Bill of Rights exemplifies the benefits of structure resulting from Copenhagen and it is another context reflecting the urgent and rising potential of our society to formulate a single climate action. It is an actually accurate description of the ongoing emissions plaguing the global economy. Link: http://www.globalwatch.org/... 2007/05/Lita v '? Summary. '20's of Truth' The wife produced, and narrated, her own journey, the success of A Million to a Time campaign, which ABC subsequently canceled, and a second show the network upgraded for this month. According to Daily Mail, Bell wrote in their post to promote the show that she's been ``embattled'' by ...
\end{gensample}

\begin{gensample}{256}{93.433}{5.1649}
The display that students are expected to be introduced by BND isn't more than soap. It has made its ugly head in the nation's confusing policy fights but the student group has changed up its plans and even changed its mind about inviting them on display. It agreed in March with U.S. Court Judge Watson -- in 2013 -- in a similar case with the heads of multiple school districts that the display violates First Amendment free speech. Obviously, schools won their lawsuit in December ruling that the school can no longer tolerate rebellion and the cultural behavior that BND and his group are encouraging, ...
\end{gensample}

\subsection{\textsc{UnifusionGPT}-M samples}
\label{app:qual-medium}

We present ten approximately 100-token samples generated by \textsc{UnifusionGPT}-M using 256 sampling steps.

\newenvironment{mediumsample}[1]{%
  \par\medskip
  \noindent\begin{minipage}{\linewidth}
  \textbf{Sample #1}\par
  \vspace{2pt}\small
}{%
  \end{minipage}\par
}

\begin{mediumsample}{1}
As attention spans are shorter, this will only become greater. When we provide online content on websites; blog posts or any other content with an audience; we need to communicate what we say in language that people actually understand. Otherwise, we can waste business hours on bulging unnecessary text with something that we will simply have difficulty saying no to. That's in the future, though. Here you've come across some of the tools for effective content marketing that you know and know what to do. That's because Content Marketing helps even the most customised site to maximise the effectiveness of its content. \ldots
\end{mediumsample}

\begin{mediumsample}{2}
Choosing a therapist assures you is that they possess wisdom and understanding to address issues that contribute to insecurities. This can change your relationship by no means going over and pasting the past---you should always view their advice as a metaphysical quest. You'll learn how therapists shed light on perspectives, beliefs, and values, allowing you to evaluate how you reasonably perceive yourself. This understanding can uncover hidden fears and anxieties, leading to the development of new viewpoints. Another aspect that can affect your future therapy role is your mode of expression. While a therapist may be able to facilitate your personal growth, they might find that you feel comfortable expressing difficult feelings or hidden thoughts. \ldots
\end{mediumsample}

\begin{mediumsample}{3}
Maybe, you are planning a healthy packed dinner, but do not have the extra time you need to prep your meal. It is not a problem if you don't have the time. Improving your dinner recipes will not only ensure they will stick for longer, it will make you happier too. When it comes to healthy snacks, there really can be no healthy frozen meals. So always cook for a healthy meal with fresh ingredients for optimum health goals. Packing can also be used as an alternative to fad diets (We got you). Instead, try finding out what type of meal option suits you. To get the very best deals on stuff, check out our keto channel. \ldots
\end{mediumsample}

\begin{mediumsample}{4}
Typically, low pressure will loosen up and seal in ragged and crushed snow leaving obstacles less visible when shoveled. Make sure snow has been cleared It may seem surprising to many but to us, you are already zoned for the chilly temperatures. Regardless if it is a pre-planned detour or a thoughtful exit, getting a clean vehicle is a great way to ensure that while you aren't blocking anyone at all, you won't have to run into surprises. The final tip on our tip list to have before you hit town is that you clear the areas of physical waste whether from leaves shoveled or pieces of snow. \ldots
\end{mediumsample}

\begin{mediumsample}{5}
Our range revolves around three most prominent products, which are covered in detail, and which we will be proud to show customers on display. Could you imagine a window box that takes high gloss profiles and quirkiness to the next level? With these quality profiles, matched with the unique Venetec sashes and trim perfectly, your projects will be even more easily successful! The first impression of your window box, perfectly finished and matching, you will shake your heads! You may have already noticed that your favourite design is your colour for years, and perhaps you just want to try something completely different entirely. Your typical diversity of door designs is many; only the must haves, and the decorations rule. \ldots
\end{mediumsample}

\begin{mediumsample}{6}
The most important aspect of reaching goals is tracking your loved one's weight plan and diet plan. Journal keepers will not only provide focus and informed decision-making but also help you stay on the dieting plan path. Obesity is a significant health issue, and staying on track with a digital health diary for parents serves as a guide to keeping eating habits on track. By creating a digital food plan, where you can plan and plan precise and detailed meals will track or change their eating plans over time, your life will become smoother, and you can track your progress at your convenience with a handy digital diary. \ldots
\end{mediumsample}

\begin{mediumsample}{7}
We hadn't quite had young kids and now, we were settled. In the beginning, we started off in danger of resetting with my focus. The expectation of a life of many responsibilities, of having many things that going on at once, especially with the kids, getting married, money, etc. seemed almost unrealistic. My biggest turning point came from mom who advised that my goal primarily changed to time, responsibility and early retirement. However, my perspective, and the fact that the end of time could change, I knew I needed some areas to work on. While that might be a simple as a meeting with a friend or relative who shared my recent ``where is your work'' scenario, I realized that there were several other aspects of life I could successfully tackle. \ldots
\end{mediumsample}

\begin{mediumsample}{8}
It is sad that Christmas is now over one year on Thursday. We will see the whole month down from the evening of December 5, and can't believe that time at the very beginning and the very end. Try to celebrate the lows and highs of a month of chocolate, oranges and scary poppins. I teamed up with GRITTNESS recipe team to create my first ever, cake in a bag treat at our home. Be ready to lift a glass whenever you need to celebrate: nobody will ever have a Christmas like you are without dessert. With wonderful sounds, flavors and sweets, Christmas is here. Illustrations and recipes Careful cooks, my mother would spoon her dessert to her mama, or just place in a bowl for me. \ldots
\end{mediumsample}

\begin{mediumsample}{9}
Bookwalk Creative is a family thrift shop that has a fun collection of fun, vintage goodies for the home. From clothes to bags, laundry bags (we love!), to books to baskets, it's a great outlet for whimsical rangeting projects and cute gifts and gift boxes. Many of these beautiful supplies we got pre-loaded with my matching gingham ribbon and with a customized backing with a matching color of gingham. I'm already making the bag made of high-quality kraft paper with a ribbon and made me buttons to make a quick birthday present. Take a closer look at these cute bag items. Family Home Panty Bag I have been secretly crushing on quilting bags for some time! \ldots
\end{mediumsample}

\begin{mediumsample}{10}
One of the main reasons a customer will get faulty machinery is its defect or condition. Customers cannot see that they are in a well maintained facility. It may be very expensive to pay them back and also later when customers need to ring you for maintenance or repairs. The lack of material or other services or post-sales work also compromises the aftersales service you develop with customers and ensures a loyal client base. Packing and transporting your goods alone can be too much of a hassle. It can also cost up to \texteuro100 to use moving companies to transport your belongings to the storage. If you don't use a specialist company and book the transport there could be a long term financial loss. \ldots
\end{mediumsample}

\endgroup


\begingroup
\small
\begin{thebibliography}{32}

\bibitem[Austin et~al.(2021)]{austin2021structured}
Jacob Austin, Daniel~D. Johnson, Jonathan Ho, Daniel Tarlow, and Rianne van~den Berg.
\newblock Structured denoising diffusion models in discrete state-spaces.
\newblock In \emph{Advances in Neural Information Processing Systems}, 34:17981--17993, 2021.

\bibitem[Bisk et~al.(2020)]{bisk2020piqa}
Yonatan Bisk, Rowan Zellers, Ronan Le~Bras, Jianfeng Gao, and Yejin Choi.
\newblock {PIQA}: Reasoning about physical commonsense in natural language.
\newblock In \emph{Proceedings of the AAAI Conference on Artificial Intelligence},
34(5):7432--7439, 2020.

\bibitem[Campbell et~al.(2022)]{campbell2022continuous}
Andrew Campbell, Joe Benton, Valentin De~Bortoli, Thomas Rainforth, George Deligiannidis, and Arnaud Doucet.
\newblock A continuous time framework for discrete denoising models.
\newblock In \emph{Advances in Neural Information Processing Systems}, 35:28266--28279, 2022.

\bibitem[Cobbe et~al.(2021)]{cobbe2021gsm8k}
Karl Cobbe, Vineet Kosaraju, Mohammad Bavarian, Mark Chen, Heewoo Jun, Lukasz Kaiser, Matthias Plappert,
Jerry Tworek, Jacob Hilton, Reiichiro Nakano, Christopher Hesse, and John Schulman.
\newblock Training verifiers to solve math word problems.
\newblock \emph{arXiv preprint arXiv:2110.14168}, 2021.

\bibitem[Gat et~al.(2024)]{gat2024discrete}
Itai Gat, Tal Remez, Neta Shaul, Felix Kreuk, Ricky T.~Q. Chen, Gabriel Synnaeve, Yossi Adi, and
Yaron Lipman.
\newblock Discrete flow matching.
\newblock In \emph{Advances in Neural Information Processing Systems}, 37:133345--133385, 2024.

\bibitem[Garg et~al.(2025)]{garg2025learnedorder}
Prateek Garg, Bhavya Kohli, and Sunita Sarawagi.
\newblock Masked diffusion models are secretly learned-order autoregressive models.
\newblock \emph{arXiv preprint arXiv:2511.19152}, 2025.

\bibitem[Gong et~al.(2025)]{gong2025scaling}
Shansan Gong, Shivam Agarwal, Yizhe Zhang, Jiacheng Ye, Lin Zheng, Mukai Li, Chenxin An, Peilin Zhao,
Wei Bi, Jiawei Han, Hao Peng, and Lingpeng Kong.
\newblock Scaling diffusion language models via adaptation from autoregressive models.
\newblock In \emph{International Conference on Learning Representations}, 2025.

\bibitem[Gulrajani and Hashimoto(2023)]{gulrajani2023likelihood}
Ishaan Gulrajani and Tatsunori~B. Hashimoto.
\newblock Likelihood-based diffusion language models.
\newblock In \emph{Advances in Neural Information Processing Systems}, 36:16693--16715, 2023.

\bibitem[Ho et~al.(2020)]{ho2020denoising}
Jonathan Ho, Ajay Jain, and Pieter Abbeel.
\newblock Denoising diffusion probabilistic models.
\newblock In \emph{Advances in Neural Information Processing Systems}, 33:6840--6851, 2020.

\bibitem[Hoogeboom et~al.(2022)]{hoogeboom2022autoregressive}
Emiel Hoogeboom, Alexey~A. Gritsenko, Jasmijn Bastings, Ben Poole, Rianne van~den Berg, and Tim Salimans.
\newblock Autoregressive diffusion models.
\newblock In \emph{International Conference on Learning Representations}, 2022.

\bibitem[Li et~al.(2026{\natexlab{a}})]{neuralctmc2026}
Jingyuan Li, Xiaoyi Jiang, Fukang Wen, Wei Liu, Renqian Luo, Yi Zhu, Zuoqiang Shi, and Pipi Hu.
\newblock Neural continuous-time Markov chain: Discrete diffusion via decoupled jump timing and direction.
\newblock \emph{arXiv preprint arXiv:2604.15694}, 2026.

\bibitem[Li et~al.(2026{\natexlab{b}})]{m2s2026}
Jingyuan Li, Xiaoyi Jiang, Yixuan Jiang, Wei Liu, Yi Zhu, Zuoqiang Shi, and Pipi Hu.
\newblock Mean-to-score discrete diffusion: Posterior-mean denoisers for score entropy.
\newblock \emph{arXiv preprint arXiv:2607.21372}, 2026.

\bibitem[Lou et~al.(2024)]{lou2024discrete}
Aaron Lou, Chenlin Meng, and Stefano Ermon.
\newblock Discrete diffusion modeling by estimating the ratios of the data distribution.
\newblock In \emph{Proceedings of the 41st International Conference on Machine Learning},
volume 235 of \emph{Proceedings of Machine Learning Research}, pages 32819--32848, 2024.

\bibitem[Nie et~al.(2025)]{nie2025llada}
Shen Nie, Fengqi Zhu, Zebin You, Xiaolu Zhang, Jingyang Ou, Jun Hu, Jun Zhou, Yankai Lin, Ji-Rong Wen,
and Chongxuan Li.
\newblock Large language diffusion models.
\newblock \emph{arXiv preprint arXiv:2502.09992}, 2025.

\bibitem[Ou et~al.(2025)]{ou2025your}
Jingyang Ou, Shen Nie, Kaiwen Xue, Fengqi Zhu, Jiacheng Sun, Zhenguo Li, and Chongxuan Li.
\newblock Your absorbing discrete diffusion secretly models the conditional distributions of clean data.
\newblock In \emph{International Conference on Learning Representations}, 2025.

\bibitem[Penedo et~al.(2024)]{penedo2024fineweb}
Guilherme Penedo, Hynek Kydl\'i\v{c}ek, Loubna Ben~Allal, Anton Lozhkov, Margaret Mitchell, Colin Raffel,
Leandro von Werra, and Thomas Wolf.
\newblock The {FineWeb} datasets: Decanting the web for the finest text data at scale.
\newblock In \emph{Advances in Neural Information Processing Systems}, 37:30811--30849, 2024.

\bibitem[Radford et~al.(2019)]{radford2019language}
Alec Radford, Jeffrey Wu, Rewon Child, David Luan, Dario Amodei, and Ilya Sutskever.
\newblock Language models are unsupervised multitask learners.
\newblock Technical report, OpenAI, 2019.

\bibitem[Sahoo et~al.(2024)]{sahoo2024simple}
Subham~Sekhar Sahoo, Marianne Arriola, Yair Schiff, Aaron Gokaslan, Edgar Marroquin, Justin~T. Chiu,
Alexander Rush, and Volodymyr Kuleshov.
\newblock Simple and effective masked diffusion language models.
\newblock In \emph{Advances in Neural Information Processing Systems}, 37:130136--130184, 2024.

\bibitem[Sahoo et~al.(2025)]{sahoo2025diffusion}
Subham~Sekhar Sahoo, Justin Deschenaux, Aaron Gokaslan, Guanghan Wang, Justin~T. Chiu, and
Volodymyr Kuleshov.
\newblock The diffusion duality.
\newblock In \emph{Forty-second International Conference on Machine Learning}, 2025.

\bibitem[Sakaguchi et~al.(2021)]{sakaguchi2021winogrande}
Keisuke Sakaguchi, Ronan Le~Bras, Chandra Bhagavatula, and Yejin Choi.
\newblock {WinoGrande}: An adversarial winograd schema challenge at scale.
\newblock \emph{Communications of the ACM}, 64(9):99--106, 2021.

\bibitem[Sap et~al.(2019)]{sap2019socialiqa}
Maarten Sap, Hannah Rashkin, Derek Chen, Ronan Le~Bras, and Yejin Choi.
\newblock {Social IQa}: Commonsense reasoning about social interactions.
\newblock In \emph{Proceedings of the 2019 Conference on Empirical Methods in Natural Language
Processing and the 9th International Joint Conference on Natural Language Processing}, pages
4463--4473, 2019.

\bibitem[Shi et~al.(2024)]{shi2024simplified}
Jiaxin Shi, Kehang Han, Zhe Wang, Arnaud Doucet, and Michalis~K. Titsias.
\newblock Simplified and generalized masked diffusion for discrete data.
\newblock In \emph{Advances in Neural Information Processing Systems}, 37:103131--103167, 2024.

\bibitem[Song et~al.(2021)]{song2021score}
Yang Song, Jascha Sohl-Dickstein, Diederik~P. Kingma, Abhishek Kumar, Stefano Ermon, and Ben Poole.
\newblock Score-based generative modeling through stochastic differential equations.
\newblock In \emph{International Conference on Learning Representations}, 2021.

\bibitem[Suzgun et~al.(2022)]{suzgun2022challenging}
Mirac Suzgun, Nathan Scales, Nathanael Sch\"arli, Sebastian Gehrmann, Yi~Tay, Hyung~Won Chung,
Aakanksha Chowdhery, Quoc~V. Le, Ed~H. Chi, Denny Zhou, and Jason Wei.
\newblock Challenging {BIG-Bench} tasks and whether chain-of-thought can solve them.
\newblock \emph{arXiv preprint arXiv:2210.09261}, 2022.

\bibitem[Uria et~al.(2014)]{uria2014deep}
Benigno Uria, Iain Murray, and Hugo Larochelle.
\newblock A deep and tractable density estimator.
\newblock In \emph{Proceedings of the 31st International Conference on Machine Learning},
volume 32 of \emph{Proceedings of Machine Learning Research}, pages 467--475, 2014.

\bibitem[von R{\"u}tte et~al.(2025)]{vonrutte2025gidd}
Dimitri von R{\"u}tte, Janis Fluri, Yuhui Ding, Antonio Orvieto, Bernhard Sch{\"o}lkopf, and Thomas Hofmann.
\newblock Generalized interpolating discrete diffusion.
\newblock In \emph{Proceedings of the 42nd International Conference on Machine Learning},
volume 267 of \emph{Proceedings of Machine Learning Research}, pages 61810--61843, 2025.

\bibitem[Ye et~al.(2025)]{ye2025dream}
Jiacheng Ye, Zhihui Xie, Lin Zheng, Jiahui Gao, Zirui Wu, Xin Jiang, Zhenguo Li, and Lingpeng Kong.
\newblock Dream 7{B}: Diffusion large language models.
\newblock \emph{arXiv preprint arXiv:2508.15487}, 2025.

\bibitem[Zellers et~al.(2019)]{zellers2019hellaswag}
Rowan Zellers, Ari Holtzman, Yonatan Bisk, Ali Farhadi, and Yejin Choi.
\newblock {HellaSwag}: Can a machine really finish your sentence?
\newblock In \emph{Proceedings of the 57th Annual Meeting of the Association for Computational
Linguistics}, pages 4791--4800, 2019.

\end{thebibliography}
\endgroup
\end{document}